\documentclass[letterpaper,11pt,reqno]{amsart}

\makeatletter
\usepackage{amssymb}
\usepackage{latexsym}
\usepackage{amsbsy}
\usepackage{amsfonts}
\usepackage{hyperref}
\usepackage{graphicx}
\usepackage{scalerel}
\usepackage{enumerate}
\usepackage{enumitem}
\usepackage{mathtools}
\usepackage{color}
\usepackage{subcaption}
\usepackage{booktabs}
\usepackage{threeparttable}

\usepackage{tikz}

\def\marginpar#1{\ignorespaces}

\DeclareMathOperator\erf{erf}

\newtheorem{theorem}{Theorem}[section]

\newtheorem{proposition}[theorem]{Proposition}

\numberwithin{equation}{section}
\makeatother
\begin{document}
\title[]{Tweedie's formulae and diffusion generative models beyond Gaussian}

\author[Wenpin Tang]{{Wenpin} Tang}
\address{Department of Industrial Engineering and Operations Research, Columbia University.
} \email{wt2319@columbia.edu}

\author[Nizar Touzi]{{Nizar} Touzi}
\address{Department of Finance and Risk Engineering, New York University.
} \email{nt2635@nyu.edu}

\author[Zikun Zhang]{{Zikun} Zhang}
\address{Department of Industrial Engineering and Operations Research, Columbia University.
} \email{zz3367@columbia.edu}

\author[Xun Yu Zhou]{{Xun Yu} Zhou}
\address{Department of Industrial Engineering and Operations Research, Columbia University.
} \email{xz2574@columbia.edu}

\date{\today}
\begin{abstract}
Diffusion models have achieved remarkable success in generating samples from unknown data distributions. Most popular stochastic differential equation–based diffusion models perturb the target distribution by adding Gaussian noise, transforming it into a simple prior, and then use denoising score matching, a consequence of Tweedie’s formula, to learn the score function and generate clean samples from noise. However, non-Gaussian diffusion models with state-dependent diffusion coefficient have been largely underexplored, as have the corresponding Tweedie’s formulae. In this work, we extend Tweedie’s formula to important non-Gaussian processes, including geometric Brownian motion (GBM), squared Bessel (BESQ) processes, and Cox-Ingersoll-Ross (CIR) processes, thereby yielding the corresponding denoising score-matching objectives.
We then apply the derived formulae to image and financial time series generation using GBM- and CIR-based diffusion models, and to empirical Bayes estimation under the BESQ setting.
The reported experimental results demonstrate the potential of non-Gaussian models. 

\end{abstract}

\maketitle

\section{Introduction}

\quad Diffusion models are a family of generative models
that
genuinely create samples from unknown  target distributions \cite{Ho20, Song19, Song20}.
They underpin the recent success
in text-to-image creators
such as DALL·E 2 \cite{ramesh2022hierarchical}, Stable Diffusion \cite{Rombach22}, and Flux \cite{Flux},
in text-to-video generators such as Sora \cite{Sora}, Make-A-Video \cite{Singer22}, Seedance \cite{gao2025seedance}, and Veo \cite{Veo},
and in diffusion large language models
such as Mercury \cite{KK25}, LLaDA \cite{Nie25}, Dream \cite{YX25}, and WeDLM \cite{liu2025wedlm}.
Recently, diffusion models have also been applied to other fields, including operations research \cite{LZ24, ZY26}
and finance \cite{AB25, GZZ25, guo2025diffusion} for tabular data generation.

\quad The idea of diffusion generative models relies on a forward--backward procedure:
\begin{itemize}[itemsep = 3 pt]
\item
{\em Forward process}:
starting from a training sample of the target distribution $X_0 \sim p_{\mathrm{data}}(\cdot)$,
the model gradually adds noise
to transform the signal into noise
$X_0 \to  \cdots \to X_T \sim p_{\mathrm{noise}}(\cdot)$.
\item
{\em Backward process}:
 start with the noise $X_T \sim p_{\mathrm{noise}}(\cdot)$,
 and reverse the forward process
 to recover the signal from noise
 $X_T \to \cdots \to X_0 \sim p_{\mathrm{data}}(\cdot)$.
\end{itemize}
The backward process is also termed as diffusion generative sampling or inference.
In this paper, we adopt the continuous-time formulation,
where the forward process $\{X_t\}_{0 \le t \le T}$ is governed by a stochastic differential equation (SDE):
\begin{equation*}
\mathrm{d}X_t = b(t, X_t) \,\mathrm{d}t + \sigma(t, X_t) \,\mathrm{d}W_t, \quad X_0 \sim p_{\mathrm{data}}(\cdot).
\end{equation*}
As will be detailed in Section \ref{sc2},
the key to diffusion generative sampling hinges on the  {\em score function} \footnote{In other words, the score function is the logarithmic derivative of the probability density of $X_t$.}:
\begin{equation*}
\nabla \log p(t,x): = \nabla \log \left( \frac{\mathrm{d}}{\mathrm{d}x} \mathbb{P}(X_t \in \mathrm{d}x) \right).
\end{equation*}
Learning or estimating this score function is often referred to as {\em score-matching}.

\quad For most existing diffusion generative models,
e.g., variance exploding (VE) and variance preserving (VP) models \cite{Song20},
the diffusion coefficient $\sigma(t,x) = \sigma(t) I$ is only time-dependent,
and the drift parameter $b(t,x) = f(t) x$ is linear in the state (or space) variable.\footnote{For the VE model, $b(t,x) = 0$ and $\sigma(t, x) = \sigma(t) I$,
with $\sigma(t) = a\sqrt{t}$ or $a b^t$ for some $a, b  > 0$.
For  the VP model, $b(t,x) = -\alpha(t) x$ and $\sigma(t,x) = \sqrt{2\alpha(t)} I$,
with $\alpha(t) = a t + b$ for some $a, b > 0$. In both cases, while the state $X_t$ is typically high-dimensional, noises are independently added to each component of $X_t$. }
In this case, the process $\{X_t\}_{0 \le t \le T}$ is Gaussian,
and there is a systematic way to learn the score function $\nabla \log p(t,x)$
via {\em Tweedie's formula} \cite{Efron11, Robbins56}.

\quad Now we briefly explain Tweedie's formula,
which first appeared in correspondence between Herbert Robbins and Maurice Tweedie,
and was rediscovered in various contexts \cite{Mas75, Miya61, PS92, Pol91}.\footnote{The formula is also known as Eddington--Tweedie formula \cite{IS25} or Masreliez's theorem \cite{Mas75}.}
It was later popularized by Bradley Efron in the context of empirical Bayes estimation to tackle selection bias in genome analysis \cite{Efron11, EZ11}; see \cite{IS25} for comprehensive discussions of empirical Bayes and Tweedie's formula.
Let
\begin{equation}
\label{eq:prior}
U \sim \nu(\cdot) \quad \mbox{and} \quad V\,|\,U \sim  \mathcal{N}(U, \sigma^2 I),
\end{equation}
where $\mathcal{N}(\mu, \Sigma)$ denotes a Gaussian random vector
with mean $\mu$ and covariance matrix $\Sigma$.
Denote by $p_V(\cdot)$ the (marginal) distribution of $V$. Then Tweedie's formula yields
\begin{equation}
\label{eq:TweedieGau}
\mathbb{E}(U \,| \, V = v) = v + \sigma^2 \nabla \log p_V(v).
\end{equation}
Notably, the equation \eqref{eq:TweedieGau} does not involve the ``prior" $\nu(\cdot)$,
and can be used to learn the score function of Gaussian diffusion models.
To illustrate, consider the VE model $\mathrm{d}X_t = \sigma(t) \,\mathrm{d} W_t$, $X_0 \sim p_{\mathrm{data}}(\cdot)$.
Specializing to $U = X_0$ and $V = X_t$ (so $\nu(\cdot)$ and $\sigma^2$ in \eqref{eq:prior} are $p_{\mathrm{data}}(\cdot)$ and $\int_0^t \sigma^2(s)\,\mathrm{d}s$, respectively) yields
\begin{equation}
\label{eq:VETweedie}
\nabla \log p_{\mathrm{VE}}(t,X_t) = \frac{\mathbb{E}(X_0 \,|\, X_t) - X_t }{\int_0^t \sigma^2(s)\,\mathrm{d}s} \quad a.s.,
\end{equation}
under suitable integrability conditions on $X_0 \sim p_{\mathrm{data}}(\cdot)$.
Next, we sample $(X_0, X_t)$ according to
$X_0 \sim p_{\mathrm{data}}(\cdot)$,
$X_t \, |\, X_0 \sim \mathcal{N}\left(X_0, (\int_0^t \sigma^2(s)\,\mathrm{d}s)I \right)$,
and regress $X_0$ over $X_t$ to learn the score function
because $\mathbb{E}(X_0 \,|\, X_t)$ is the $L^2$ projection of $X_0$ over $X_t$ (see \cite{SH19} for related discussions).
More precisely, letting $\{s^{\mathrm{VE}}_\theta(\cdot,\cdot)\}_\theta$ be a parametrized family approximating the score function,
we aim to solve the problem:
\begin{equation*}
\min_\theta \mathbb{E}_{t \sim \mathcal{U}[0,T], \, (X_0, X_t)} \left| s^{\mathrm{VE}}_\theta(t,X_t) + \frac{X_t - X_0}{\int_0^t \sigma^2(s) \,\mathrm{d}s }\right|^2,
\end{equation*}
where $\mathcal{U}[0,T]$ denotes the uniform distribution over $[0,T]$,
and $(X_0, X_t)$ is sampled as described.
The optimization problem is known as {\em denoising score-matching},
which a priori chooses the $L^2$ loss to achieve
$\mathbb{E}_{t \sim \mathcal{U}[0,T], \, X_t}|s^{\mathrm{VE}}_{\theta_*}(t,X_t) - \nabla \log p_{\mathrm{VE}}(t,X_t)|^2 \approx 0$
if $\{s^{\mathrm{VE}}_\theta(\cdot,\cdot)\}_\theta$ is sufficiently rich.\footnote{Denoising score-matching requires to solve:
\begin{equation*}
    \min_\theta \mathbb{E}_{t \sim \mathcal{U}[0,T], \, (X_0, X_t)} \left| s_\theta(t,X_t) - \nabla \log p(t,X_t \,|\, X_0) \right|^2,
\end{equation*}
which can be shown to be equivalent to the problem
$\min_\theta \mathbb{E}_{t \sim \mathcal{U}[0,T], \, X_t} \left| s_\theta(t,X_t) - \nabla \log p(t,X_t) \right|^2$ \cite{tang2024score, Vi11}.
While Tweedie's formula provides a theoretical justification for denoising score-matching, the reverse is not true: denoising score-matching does not inherently imply Tweedie's formula.
Another consequence is the sample complexity of denoising score-matching: 
the classical result \cite{Stone80, Stone82} shows that if $\nabla \log p(t,x)$ is $p$-times differentiable,
then 
$\mathbb{E}_{t \sim \mathcal{U}[0,T], \, X_t}|s_{\theta_*}(t,X_t) - \nabla \log p(t,X_t)|^2 \lesssim m^{-\frac{2p}{2p+d}}$,
where $m$ is the number of samples.
Refer to \cite{DF25, FGL25, Oko23} for sharper results on diffusion score estimation via low-dimensional adaptation.}
It is worth noting that Tweedie’s formula suggests that the $L^2$ loss is a natural choice for denoising score matching, since the conditional expectation corresponds to the optimal $L^2$ regression estimator.
Tweedie's formula and the corresponding denoising score-matching for other diffusion models
with $\sigma(t,x) = \sigma(t) I$ and $b(t,x) = f(t) x$ can be derived similarly.
In fact, the score function of any such model can be obtained from that of the VE model with $\sigma(t) = \sqrt{2 t}$
by a space--time reparametrization \cite{Karras22}; see also \cite[Section 5.1]{tang2024score} and \cite[Section 4.2]{TZ24}.

\medskip
{\bf Contributions.} As discussed earlier, Tweedie’s formula~\eqref{eq:TweedieGau} is formulated for Gaussian distributions, and
denoising score matching for existing diffusion models predominantly relies on their Gaussian structure, which enables the direct application of Tweedie’s formula. However, using non-Gaussian models, such as (time-dependent) geometric Brownian motion (GBM) and Bessel-type processes, may be advantageous in certain generative AI tasks, which calls for a study on Tweedie’s formula and the corresponding denoising score matching beyond Gaussian,
especially for those with state-dependent diffusion coefficients.
The main contribution of this paper is to carry out this study. 
Taking the one-dimensional setting for instance,
the key to generalize Tweedie's formula is the following simple observation (see Proposition \ref{eq:key2}):
\begin{equation}
\label{eq:key}
\sigma^2(t,x) \nabla \log p (t,x) + 2 \sigma(t,x) \partial_x \sigma(t,x)= b(t,x) + \lim_{\varepsilon \to 0} \frac{1}{\varepsilon} \mathbb{E}\left(X_{t - \varepsilon} - X_t \,|\, X_t = x\right).
\end{equation}
The identity \eqref{eq:key} enables us to derive Tweedie’s formula for non-Gaussian diffusion models with state-dependent diffusion coefficients and, consequently, the corresponding denoising score-matching objectives. Table~\ref{tab:tweedie} summarizes explicit Tweedie's formulae for several non-Gaussian models we derive in this paper which, to our best knowledge, is novel.\footnote{As mentioned earlier, as noises are added component-wise to a multi-dimensional diffusion process, we only need those formulae in a scalar space.} Notably, we emphasize that the flexible choice of time-dependent drift and diffusion coefficients is crucial to the empirical success of the corresponding diffusion models. In particular, specializing to $t = 1$ provides further extensions to  Efron's generalization of Tweedie's formula to the exponential family
with linear sufficient statistics and in the canonical form \cite{Efron11}.\footnote{The exponential family is of the form $p(V \,|\, U) = h(V) \exp(T(V) \eta(U)-A(U))$, where $T(V)$ is the sufficient statistics.
It is said to be in the canonical form if $\eta(U) = U$. Efron \cite{Efron11} generalizes Tweedie's formula to the exponential family with $T(V) \propto V$ and $\eta(U) = U$, which however does not include GBM and Bessel processes.}
As for applications, we use non-Gaussian models
(with score matching via Tweedie's formula)
to perform image generation, financial time series generation, and empirical Bayes estimation,
which have been largely underexplored in prior work.

\begin{table}[!h]
\caption{Tweedie's Formulae for Gaussian and Non-Gaussian Models}
\label{tab:tweedie}
\centering
\renewcommand{\arraystretch}{1.5}
\setlength{\tabcolsep}{5pt}
\resizebox{\textwidth}{!}{
\begin{threeparttable}
\begin{tabular}{l l l l}
\toprule
       Process & $b(t,x)$ & $\sigma(t,x)$ & Score function $\nabla \log p(t,x)$ \\
       \midrule
       VE   & $0$ & $\sigma(t)$ & $ \Sigma^{-2}(t)(\mathbb{E}(X_0\,|\, X_t =x) - x)$ \\
       VP  & $-\alpha(t)x$ & $\sqrt{2\alpha(t)}$ & $(1-e^{-2A(t)})^{-1} ( e^{-A(t)}\, \mathbb{E}(X_0 \,|\, X_t = x) -x)$ \\
      GBM   & $\mu(t) x$ & $\sigma(t) x$ & $\left(\frac{U(t)}{\Sigma^2(t)} - \frac{3}{2} \right)\frac{1}{x} - \frac{1}{\Sigma^2(t)}\frac{\log x}{x} + \frac{1}{\Sigma^2(t)} \frac{1}{x}\mathbb{E}(\log X_0 \,|\, X_t = x)$  \\
      BESQ   & $2(\nu+1)$ & $2\sqrt{x}$ & $\frac{\nu}{x} -\frac{1}{2t} + \frac{1}{2t \sqrt{x}}\mathbb{E}(\sqrt{X_0} I_{\nu+1}(\frac{\sqrt{xX_0}}{t})/I_{\nu}(\frac{\sqrt{xX_0}}{t}) \,|\, X_t = x)=:s^\mathrm{BESQ}_\nu(t,x)$ \\
      CIR & $\alpha(t)(\mu(t)-x)$ & $\sigma(t)\sqrt{x}$ & $e^{A(t)} s_{\frac{2 \alpha(t) \mu(t)}{\sigma^2(t)} - 1(\equiv \nu)}^{\mathrm{BESQ}}\left(\frac{1}{4}\int_0^t \sigma^2(s)e^{A(s)}\,\mathrm{d}s,  e^{A(t)} x\right)$  \\
      CEV   &  $\mu(t) x$ & $\sigma(t) x^\beta$ & $\frac{-2 \beta + 1}{x}
-\frac{2(\beta - 1) e^{2 (\beta- 1)U(t)}}{x^{2 \beta - 1}}s^{\mathrm{BESQ}}_{\frac{1}{2(\beta - 1)}}\left((\beta-1)^2\int_0^t \sigma^2(s)e^{2(\beta-1)U(s)}\,\mathrm{d}s, e^{2 (\beta- 1)U(t)} x^{-2 (\beta - 1)}\right)$ \\
      BES(3)  & $1/x$ & $1$ & $\frac{1}{x} + \frac{1}{t} \mathbb{E}( (X_0 - x) \coth(\frac{x X_0}{t} ) \,|\, X_t = x )$ \\
      \bottomrule
    \end{tabular}
    \end{threeparttable}
    }
    \begin{tablenotes}
\footnotesize
\item Here, for all $t\geq 0$, $\Sigma^2(t):=\int_0^t \sigma^2(s)\,\mathrm{d}s$, $A(t):=\int_0^t \alpha(s)\,\mathrm{d}s$, and $U(t):=\int_0^t \mu(s)\,\mathrm{d}s$.
\item $I_\nu(\cdot)$ denotes the modified Bessel function of the first kind of order $\nu$.
\item
VE and VP are Gaussian models;
GBM, BESQ, CIR, CEV and BES(3) are (positive) non-Gaussian models.
\end{tablenotes}
\end{table}

\medskip
{\bf Related Works.} SDE/score-based continuous diffusion models are formally introduced in \cite{Song20}. Their empirical success across various applications relies on the careful design of the drift and diffusion coefficients in the forward SDE. State-of-the-art VE and VP models \cite{Karras22, Song20} are Gaussian-based: they progressively add Gaussian noise to the original data distribution, use state-{\it independent} diffusion coefficients for simplicity and numerical stability, and generate the reverse sampling process from a Gaussian noise distribution.

\quad Non-Gaussian SDE-based diffusion models are often associated with modeling discrete or categorical data on the probability simplex. \cite{richemond2022categorical} proposes a multi-dimensional CIR process as the forward noising process for simplex diffusions, which leads to a Dirichlet prior after normalizing the limiting multivariate Gamma distribution.
Though the authors note that the CIR process is well suited for simplex diffusions over categorical data due to its positivity and the existence of a limiting distribution and comment on numerical simulations, they do not actually conduct numerical experiments. \cite{floto2023diffusion} applies an additive logistic transformation to the Ornstein--Uhlenbeck process to construct a positive forward SDE for simplex diffusions, derives the corresponding score-matching objective, and demonstrates their approach using MNIST images with pixels discretized into three categories. \cite{avdeyev2023dirichlet} proposes a Dirichlet diffusion score model for discrete and categorical data by constructing a multivariate diffusion process on the probability simplex. The forward process converges to a Dirichlet distribution and is built from independent univariate Jacobi diffusion processes via a stick-breaking construction. 
Inspired by Dale's law, 
\cite{Shetty25} considers GBM with constant coefficients
as a multiplicative forward process to model
non-negative data and proposes a new multiplicative score-matching loss to train
the model, showing the promise and applicability of
the non-Gaussian models to datasets and domains where multiplicative noise is preferred. Their multiplicative score-matching loss for GBM differs from ours in that it is obtained by directly multiplying the state variable with the classical denoising score-matching loss while we focus on deriving such objectives for a broader class of non-Gaussian models in a systematic way.
\cite{kim2025diffusion} puts forward  a diffusion-based generative framework for financial time series that incorporates GBM into the forward process in the price space. Although the authors observe that, under a balance between the drift and diffusion coefficients, the model reduces to additive Gaussian noise injection (i.e, a VE formulation) in the log-price space, they do not derive the corresponding score matching objectives for general GBM-based diffusion models.

\quad Our primary goal is to derive Tweedie's formulae as denoising score matching objectives for various (important) diffusion models beyond Gaussian.
That is, we express the score function $\nabla \log p(t,x)$ in terms of the conditional expectation $\mathbb{E}(g(t,x, X_0)\,|\, X_t = x)$ for some explicit function $g$.
Upon the completion of this paper, we noted a recent preprint \cite{Torres26}
that calculated $\mathbb{E}(X_0 \,|\, X_t =x)$ for additive models beyond Gaussian
via some transformations (such as differentiation and integration)  of the density function $p(t,x)$.
This type of  ``Tweedie calculus"  (see also \cite{GK17, SF19} for special cases) is different from our results because
 it only applies to additive models whereas GBM and Bessel processes are not additive, and 
 it goes ``the other way around" by seeking an expression involving the density function for the conditional expectation
rather than seeking a conditional expectation representation for the score function.

\medskip
{\bf Organization of the Paper.} The remainder of the paper is organized as follows.
Section~\ref{sc2} provides background on diffusion models, where a proof of \eqref{eq:key} is given.
In Section~\ref{sc3}, we derive Tweedie's formulae for various non-Gaussian diffusion models,
 including GBM and (squared) Bessel processes.
Numerical experiments are reported in Section~\ref{sc4}.
We conclude in Section~\ref{sc5}. Additional results on GBM-based Bayes estimation and experiments specifics are placed in the appendix.

\section{Diffusion models and Tweedie's formulae}
\label{sc2}

\quad This section provides background on diffusion models,
and presents a simple yet general approach to derive Tweedie's formulae in the context of diffusion models.
We follow closely the presentation of \cite{tang2024score}.

 \quad Consider a forward SDE:
\begin{equation}
\label{eq:SDE}
\mathrm{d}X_t = b(t, X_t) \,\mathrm{d}t + \sigma(t,X_t) \,\mathrm{d}W_t, \quad X_0 \sim p_{\mathrm{data}}(\cdot),
\end{equation}
where $\{W_t\}_{t \ge 0}$ is $n$-dimensional Brownian motion,
and $b: \mathbb{R}_+ \times \mathbb{R}^d \to \mathbb{R}^d$
and $\sigma: \mathbb{R}_+ \times \mathbb{R}^d \to \mathbb{R}^{d \times n}$
are drift and diffusion coefficients, respectively.
Some conditions on $b(\cdot, \cdot)$ and $\sigma(\cdot, \cdot)$ are required to
ensure that \eqref{eq:SDE} is well-posed (i.e. having existence and uniqueness of solution);
see standard textbooks e.g., \cite{KS91, SV79} for details about the well-posedness of SDEs.

\quad For ease of presentation,
we assume that $X_t$ has a (suitably smooth)
probability density function $p(t, \cdot)$.
The following theorem gives the time-reversal of the SDE \eqref{eq:SDE},
which lays the foundation of diffusion generative models.

\begin{theorem}
\cite{Ander82, HP86}
\label{thm:SDErev}
Denote by $a(t,x):= \sigma(t,x) \sigma(t,x)^\top$.
Under suitable conditions on $b(\cdot, \cdot)$, $\sigma(\cdot, \cdot)$ and $\{p(t, \cdot)\}_{0 \le t \le T}$,
we define
\begin{equation*}
\overline{\sigma}(t,x) = \sigma(T-t, x), \quad
\overline{b}(t,x) = -b(T-t, x) + \frac{\nabla \cdot (p(T-t,x) a(T-t,x))}{p(T-t, x)},
\end{equation*}
as well as the process $\{Y_t\}_{0 \le t \le T}$ by
\begin{equation*}
\mathrm{d} Y_t = \bar{b}(t, Y_t) \,\mathrm{d}t + \bar{\sigma}(t, Y_t ) \,\mathrm{d} B_t, \quad Y_0 \sim p(T, \cdot),
\end{equation*}
where $\{B_t\}_{0 \le t \le T}$ is a copy of Brownian motion.
Then $\{Y_t\}_{0 \le t \le T}$ and $\{X_{T-t}\}_{0 \le t \le T}$ have the same marginal distribution,
i.e., $Y$ is the time reversal of $X$ in law.
\end{theorem}

\quad As mentioned earlier,
the high-level idea of diffusion models is to recreate samples of 
the hidden target distribution from {\em noise}.
However, the initialization $Y_0 \sim p(T, \cdot)$ depends on the {\it unknown}
$p_{\mathrm{data}}(\cdot)$ in each sample generation.
One way to resolve this issue
is to replace the initialization $Y_0 \sim p(T, \cdot)$
with some noise $p_{\mathrm{noise}}(\cdot)$:
\begin{equation}
\label{eq:SDErev}
\mathrm{d} Y_t = \bar{b}(t, Y_t) \,\mathrm{d}t + \bar{\sigma}(t, Y_t ) \,\mathrm{d} B_t, \quad Y_0 \sim p_{\mathrm{noise}}(\cdot).
\end{equation}
The choice of $p_{\mathrm{noise}}(\cdot)$ is model-specific.
For instance, $p_{\mathrm{noise}}(\cdot)$ is taken as
$\mathcal{N}\left(0, (\int_0^T \sigma^2(s)\,\mathrm{d}s)I \right)$ for the VE model,
and $\mathcal{N}(0, I)$ for the VP model.
It is expected that the closer the distributions $p(T, \cdot)$ and $p_{\mathrm{noise}}(\cdot)$ are,
the closer the distribution of $Y_T$ sampled from \eqref{eq:SDErev} is to $p_{\mathrm{data}}(\cdot)$; see \cite{benton2024nearly, chen2023improved, chen2023sampling, lee2022convergence, LW23} and \cite[Section 6]{tang2024score} along with  references therein for the convergence theory of diffusion models.

\quad Since $b(\cdot, \cdot)$ and $\sigma(\cdot, \cdot)$ are chosen in advance,
all but the term $\nabla \log p(T-t, Y_t)$ in \eqref{eq:SDErev} are available.
So in order to sample the backward process \eqref{eq:SDErev},
we need to learn or estimate the score function $\nabla \log p(t,x)$,
known as {\it score-matching}, via a parameterized family $\{s_\theta(t,x)\}_\theta$.
There are several existing score-matching methods,
among which the most widely used one is denoising score matching \cite{Hyv05, Vi11}.
As explained in the introduction,
denoising score-matching is essentially equivalent to Tweedie's formula for Gaussian models.
However, Tweedie's formulae for non-Gaussian processes are absent,
leaving out important examples such as GBM and Bessel processes.
The goal of this paper is to develop a systematic approach to generalize Tweedie's formula to
include a wider class of diffusion models for potential applications,
premised upon the following result. 

\begin{proposition}
Under suitable conditions on $b(\cdot, \cdot)$, $\sigma(\cdot, \cdot)$ and $\{p(t, \cdot)\}_{0 \le t \le T}$,
we have for almost every $x$,
\begin{equation}
\label{eq:key2}
a(t, x) \nabla \log p(t,x) + \nabla \cdot a(t,x) = b(t,x) + \lim_{\varepsilon \to 0} \frac{1}{\varepsilon} \mathbb{E}\left(X_{t - \varepsilon} - X_t \,|\, X_t = x\right).
\end{equation}
\end{proposition}

\begin{proof}
By Theorem \ref{thm:SDErev}, we have:
\begin{multline*}
Y_{T-t + \varepsilon} - Y_{T-t} = \int^{T-t + \varepsilon}_{T-t} \left(-b(T-s, Y_s) + a(T-s, Y_s) \nabla \log p(s, Y_s) + \nabla \cdot a(T-s, Y_s) \right)\,\mathrm{d}s \\
+ \int^{T-t + \varepsilon}_{T-t}  \sigma(T-s, Y_s) \,\mathrm{d}B_s.
\end{multline*}
By the Lebesgue differentiation theorem, we get for almost every $x$,
\begin{equation}
\label{eq:Leb}
\lim_{\varepsilon \to 0} \frac{1}{\varepsilon} \mathbb{E}\left(Y_{T-t + \varepsilon} - Y_{T-t} \,|\, Y_{T-t} = x\right)
= -b(t, x) + a(t, x) \nabla \log p(t,x) + \nabla \cdot a(t,x).
\end{equation}
Identifying the left side of \eqref{eq:Leb} with  $\lim_{\varepsilon \to 0} \frac{1}{\varepsilon} \mathbb{E}\left(X_{t - \varepsilon} - X_t \,|\, X_t = x\right)$ yields the desired result.
\end{proof}

\quad The identity \eqref{eq:key2} provides a systematic way to derive Tweedie's formulae for general diffusion models
via the computation of $\lim_{\varepsilon \to 0} \frac{1}{\varepsilon} \mathbb{E}\left(X_{t - \varepsilon} - X_t \,|\, X_t = x\right)$:
\begin{equation*}
    \nabla \log p (t,x) = a(t,x)^{-1} \left(b(t,x) + \lim_{\varepsilon \to 0} \frac{1}{\varepsilon} \mathbb{E}\left(X_{t - \varepsilon} - X_t \,|\, X_t = x\right) - \nabla \cdot a(t,x) \right),
\end{equation*}
when $a(t,x)$ is invertible.
As shown in Section \ref{sc3},
this quantity has a closed-form expression for many important models.
Also note that the ``almost everywhere" identity \eqref{eq:key2} is suitable for denoising score matching, which is defined as an $L^2$ loss.

\section{Tweedie's Formulae}
\label{sc3}

In this section we derive Tweedie's formulae for various diffusion processes. Although the SDE
(\ref{eq:SDE}) is formulated in $\mathbb{R}^d$ with an $n$-dimensional Brownian motion, as discussed earlier diffusion generative models (including the VE and VP) always noises/denoise independently to each component of the vector state. Therefore, throughout this section we assume that $d=n=1$.

\subsection{Variance Exploding/Preserving Processes}
\label{sc31}

We start with deriving  the formulae for the processes in the Gaussian family most commonly used by the current generative diffusion models.
\subsubsection{Variance Exploding Processes}
Consider the VE process:
\begin{equation*}
\mathrm{d}X_t = \sigma(t) \,\mathrm{d}W_t, \quad X_0 \sim p_{\mathrm{data}}(\cdot),
\end{equation*}
where $\sigma(\cdot)$ is positive, continuous and bounded away from $0$. Then we have
\begin{equation}\label{VE_expression}
X_t=X_0+\sqrt{\int_0^t\sigma^2(s)\,\mathrm{d}s}\cdot Z,\quad Z\sim\mathcal{N}(0,1).
\end{equation}
Our goal is to compute the conditional expectation
$\mathbb{E}(X_{t-\varepsilon} \,|\, X_0 = z, X_t = x)$ for $0<\varepsilon \le t $.

\quad Let
$p_{\scaleto{\Sigma}{5 pt}}(x):= \frac{1}{\sqrt{2 \pi \Sigma}} \exp\left( -\frac{x^2}{ 2 \Sigma}\right)$
be the probability density of a Gaussian random variable with mean zero and variance $\Sigma>0$.
Set
\begin{equation*}
\Sigma_1:=\int_{0}^{t-\varepsilon} \sigma^2(s)\,\mathrm{d}s, \quad \Sigma_2:=\int_{t-\varepsilon}^t \sigma^2(s)\,\mathrm{d}s.
\end{equation*}
It is known (see \cite{FPY92}) that the probability density of $(X_{t - \varepsilon}\,|\, X_0 = z, X_t = x)$ is given by
\begin{equation*}
\begin{aligned}
\mathrm{d} \mathbb{P}(X_{t - \varepsilon} \in \mathrm{d}y \,|\, X_0 = z, X_t = x)/\mathrm{d}y
& = \frac{p_{\scaleto{\Sigma_1}{5 pt}}(z-y) p_{\scaleto{\Sigma_2}{5 pt}}(x-y)}{p_{\scaleto{\Sigma_1 + \Sigma_2}{5 pt}}(x-z)} \\
& \propto \exp\left(-\frac{\Sigma_1 + \Sigma_2}{2 \Sigma_1 \Sigma_2}\left(y - \frac{x\Sigma_1 + z\Sigma_2}{\Sigma_1 + \Sigma_2} \right)^2 \right).
\end{aligned}
\end{equation*}
That is, $(X_{t - \varepsilon}\,|\, X_0 = z, X_t = x)$ is a Gaussian random variable with mean $\frac{x\Sigma_1 + z\Sigma_2}{\Sigma_1 + \Sigma_2}$
and variance $\frac{\Sigma_1 \Sigma_2}{\Sigma_1 + \Sigma_2}$.
Thus,
\begin{equation*}
\mathbb{E}(X_{t-\varepsilon} \,|\, X_0 = z, X_t = x) = \frac{x\Sigma_1 + z\Sigma_2}{\Sigma_1 + \Sigma_2},
\end{equation*}
and
\begin{equation*}
\lim_{\varepsilon \to 0} \frac{1}{\varepsilon} \mathbb{E}(X_{t - \varepsilon} - X_t \,|\, X_0 = z, X_t  = x) =
(z-x) \sigma^2(t)\left(\int_0^t \sigma^2(s)\,\mathrm{d}s\right)^{-1}.
\end{equation*}
This recovers the standard Tweedie's formula:
\begin{equation}\label{Tweedie_BM}
\nabla \log p_{\sigma}^{\mathrm{VE}}(t,x) = \left(\int_0^t \sigma^2(s)\,\mathrm{d}s\right)^{-1} \left(\mathbb{E}(X_0\,|\, X_t =x) - x\right).
\end{equation}

\subsubsection{Variance Preserving Processes} For the VP process
\begin{equation*}
    \mathrm{d}X_t = -\alpha(t) X_t \,\mathrm{d} t +  \sqrt{2\alpha(t)}\,\mathrm{d}W_t, \quad X_0 \sim p_{\mathrm{data}}(\cdot),
\end{equation*}
where $\alpha(\cdot)$ is a increasing positive function, we know that
$$
X_t=e^{-\int_0^t \alpha(s)\,\mathrm{d}s} \cdot X_0+\sqrt{1-e^{-2\int_0^t \alpha(s)\,\mathrm{d}s}}\cdot Z, \quad Z\sim \mathcal{N}(0,1).
$$
Hence, by \eqref{VE_expression} and \eqref{Tweedie_BM}, we get Tweedie's formula for the VP process:
\begin{equation}\label{Tweedie_VP}
    \nabla \log p^{\mathrm{VP}}_{\alpha}(t,x) =(1-e^{-2\int_0^t \alpha(s)\,\mathrm{d}s})^{-1} ( e^{-\int_0^t \alpha(s)\,\mathrm{d}s}\, \mathbb{E}(X_0 \,|\, X_t = x) -x).
\end{equation}

\subsection{Geometric Brownian Motion}

We consider $b(t,x) = \mu(t) x$ and $\sigma(t,x) = \sigma(t) x$ with $\sigma(t)\geq 0$.
The forward process is the time-dependent geometric Brownian motion:
\begin{equation*}
\mathrm{d}X_t = \mu(t) X_t \,\mathrm{d}t + \sigma(t) X_t \,\mathrm{d}W_t, \quad X_0 \sim p_{\mathrm{data}}(\cdot),
\end{equation*}
which has a closed form solution
\begin{equation*}
X_t = X_0 \exp\left(\int_0^t\left(\mu(s)-\frac{1}{2}\sigma^2(s)\right)\,\mathrm{d}s + \int_0^t \sigma(s)\,\mathrm{d}W_s \right), \quad t \ge 0.
\end{equation*}
It is easy to see that
\begin{equation*}
\mathbb{E}(X_{t - \varepsilon} \,|\, X_0 = z, X_t = x)
= z e^{\int_0^{t-\epsilon}(\mu(s)-\frac{1}{2}\sigma^2(s))\,\mathrm{d}s }\cdot \mathbb{E}\left(e^{M_{t-\epsilon}} \,|\, M_t = \kappa(t)\right),
\end{equation*}
where
\begin{equation*}
M_t:=\int_0^t\sigma(s)\,\mathrm{d}W_s, \quad \text{and}\quad  \kappa(t): = \log\left(\frac{x}{z}\right) - \int_0^t \left(\mu(s)-\frac{1}{2}\sigma^2(s)\right)\,\mathrm{d}s, \quad \forall t\geq 0.
\end{equation*}
By the Dubins–Schwarz theorem \cite{dubins1965continuous}, there exists a copy of Brownian motion, denoted by  $\{B_t\}_{t \ge 0}$, such that $M_t=B_{\tau_t}$, where $\tau_t:=\langle M\rangle_t=\int_0^t \sigma^2(s)\,\mathrm{d}s$. Then we have
$$
\mathbb{E}\left(e^{M_{t-\epsilon}} \,|\, M_t = \kappa(t)\right)=\mathbb{E}\left(e^{B_{\tau_{t-\epsilon}}} \,|\, B_{\tau_t} = \kappa(t)\right)=\exp\left(\frac{\int_0^{t-\epsilon}\sigma^2(s)\,\mathrm{d}s}{\int_0^{t}\sigma^2(s)\,\mathrm{d}s}\left(\kappa(t)+\frac{1}{2}\int_{t-\epsilon}^t\sigma^2(s)\,\mathrm{d}s \right)\right),
$$
where we have used the fact that $(B_{\tau_{t-\epsilon}}\,|\, B_{\tau_t}=\kappa(t))\sim \mathcal{N}(\kappa(t)\tau_{t-\epsilon}/\tau_t, \tau_{t-\epsilon}(\tau_t-\tau_{t-\epsilon})/\tau_t)$ in the second equality.
Consequently, we have
\begin{align*}
\mathbb{E}(X_{t - \varepsilon} \mid X_0 = z, X_t = x)
= {}& x \exp\Bigg(
\Bigg(
\frac{\int_0^{t-\epsilon}(\mu(s)-\frac{1}{2}\sigma^2(s))\,\mathrm{d}s}
{\int_0^{t}(\mu(s)-\frac{1}{2}\sigma^2(s))\,\mathrm{d}s}
- \frac{1}{\int_0^{t}\sigma^2(s)\,\mathrm{d}s}\log\frac{x}{z} \\
&\quad + \frac{\int_0^{t-\epsilon}\sigma^2(s)\,\mathrm{d}s}
{\int_0^{t}\sigma^2(s)\,\mathrm{d}s}
\left(
\frac{-\int_0^t(\mu(s)-\frac{1}{2}\sigma^2(s))\,\mathrm{d}s}
{\int_{t-\epsilon}^{t}\sigma^2(s)\,\mathrm{d}s}
+ \frac{1}{2}
\right)
\Bigg)
\int_{t-\epsilon}^t \sigma^2(s)\,\mathrm{d}s
\Bigg)
\end{align*}
and
\begin{equation*}
\lim_{\varepsilon \to 0} \frac{1}{\varepsilon} \mathbb{E}(X_{t - \varepsilon} - X_t \,|\, X_0 = z, X_t  = x) = x\left(
-\frac{\sigma^2(t)}{\int_0^t \sigma^2(s)\,\mathrm{d}s}
\log\frac{x}{z}
-\mu(t)+\frac{\sigma^2(t)}{2}+\sigma^2(t)\frac{\int_0^t\mu(s)\,\mathrm{d}s}{\int_0^t \sigma^2(s)\,\mathrm{d}s}
\right).
\end{equation*}
We obtain Tweedie's formula for time-dependent GBM:
\begin{equation}\label{Tweedie_GBM}
\nabla \log p_{\mu, \sigma}^{\mathrm{GBM}}(t,x) = \left(\frac{\int_0^t \mu(s)\,\mathrm{d}s}{\int_0^t\sigma^2(s)\,\mathrm{d}s} - \frac{3}{2} \right)\frac{1}{x} - \frac{1}{\int_0^t\sigma^2(s)\,\mathrm{d}s}\frac{\log x}{x} + \frac{1}{\int_0^t\sigma^2(s)\,\mathrm{d}s} \frac{1}{x}\mathbb{E}(\log X_0 \,|\, X_t = x).
\end{equation}

\subsection{Squared Bessel Processes}

Here we take $b(t,x) = 2(\nu + 1)$ with $\nu > 0$, and $\sigma(t,x) = 2 \sqrt{x}$.
The forward process is governed by
\begin{equation}
\label{eq:SDEnu}
\mathrm{d}X_t = 2 (\nu + 1)\,\mathrm{d}t + 2 \sqrt{X_t} \,\mathrm{d}W_t, \quad X_0 \sim p_{\mathrm{data}}(\cdot),
\end{equation}
which is transient.
It follows from \cite[(2.i)]{PY82} (see also \cite[Chapter XI]{RY99}) that
the transition density of the process $\{X_t\}_{t \ge 0}$ is
\begin{equation*}
q(t,x,y):=\mathrm{d}\mathbb{P}(X_{t_0 + t} \in \mathrm{d}y\,|\, X_{t_0} = x)/\mathrm{d}y = \frac{1}{2t} \left( \frac{y}{x}\right)^{\frac{\nu}{2}}\exp\left(-\frac{x+y}{2t} \right)
I_\nu\left(\frac{\sqrt{xy}}{t} \right),
\end{equation*}
where $I_{\nu}(\cdot)$ is the modified Bessel function of the first kind of order $\nu$.
So the probability density of $(X_{t - \varepsilon}\,|\, X_0 = z, X_t = x)$ is given by
\begin{equation*}
\begin{aligned}
\mathrm{d} \mathbb{P}(X_{t - \varepsilon} \in \mathrm{d}y \,|\, X_0 = z, X_t = x)/\mathrm{d}y
& = \frac{q(t-\varepsilon, z,y) q(\varepsilon, y,x)}{q(t, z,x)} \\
& =\frac{\exp\left(-\frac{ty}{2 (t - \varepsilon) \varepsilon} \right)I_{\nu}\left( \frac{\sqrt{zy}}{t - \varepsilon}\right)
I_{\nu}\left( \frac{\sqrt{xy}}{\varepsilon}\right)}{\int_0^\infty \exp\left(-\frac{ty}{2 (t - \varepsilon) \varepsilon} \right)I_{\nu}\left( \frac{\sqrt{zy}}{t - \varepsilon}\right)
I_{\nu}\left( \frac{\sqrt{xy}}{\varepsilon}\right) \mathrm{d}y}.
\end{aligned}
\end{equation*}
By letting
\begin{equation}
\label{eq:FGdef}
F(a,b,c; \nu):= \int_0^\infty e^{-cu} I_{\nu}(a\sqrt{u}) I_{\nu}(b \sqrt{u})\,\mathrm{d}u \quad
\mbox{and} \quad
G(a,b,c; \nu):= -\frac{\partial F}{\partial c},
\end{equation}
we have
\begin{equation}
\label{eq:teps}
\mathbb{E}(X_{t - \varepsilon}\,|\, X_0 = z, X_t = x)
= \frac{G \left(\frac{\sqrt{z}}{t-\varepsilon}, \frac{\sqrt{x}}{\varepsilon}, \frac{t}{2(t - \varepsilon) \varepsilon}; \nu \right)}{F \left(\frac{\sqrt{z}}{t-\varepsilon}, \frac{\sqrt{x}}{\varepsilon}, \frac{t}{2(t - \varepsilon) \varepsilon}; \nu \right)}.
\end{equation}

\quad According to \cite[(10.43.28)]{DLMF} and by a change of variable $\sqrt{u} \to u$,
we get
\begin{equation}
\label{eq:Ffor}
F(a,b,c; \nu) = \frac{1}{c} \exp\left(\frac{a^2 + b^2}{4c} \right) I_\nu \left(\frac{ab}{2c}\right).
\end{equation}
Thus,
\begin{equation}
\label{eq:Gfor}
\begin{aligned}
G(a,b,c; \nu) &= \frac{1}{c^2}\exp\left(\frac{a^2 + b^2}{4c}\right)
\left( \left(1 + \frac{a^2 + b^2}{4c} \right)I_\nu\left(\frac{ab}{2c} \right) + \frac{ab}{2c}I'_\nu\left(\frac{ab}{2c} \right) \right) \\
& = \frac{1}{c^2}\exp\left(\frac{a^2 + b^2}{4c}\right)
\left( \left(1 + \frac{a^2 + b^2}{4c} + \nu \right)I_\nu\left(\frac{ab}{2c} \right) + \frac{ab}{2c}I_{\nu+1}\left(\frac{ab}{2c} \right) \right),
\end{aligned}
\end{equation}
where we use the fact that $I'_\nu(u) = \frac{\nu}{u} I_\nu(u) + I_{\nu+1}(u)$ (see \cite[(10.29.2)]{DLMF}) in the last equation.
Combining \eqref{eq:teps}, \eqref{eq:Ffor} and \eqref{eq:Gfor} yields
\begin{equation*}
\begin{aligned}
\mathbb{E}(X_{t - \varepsilon}\,|\, X_0 = z, X_t = x)
& = \frac{2(t - \varepsilon) \varepsilon}{t}
\left( 1 + \nu + \frac{\varepsilon^2 z + (t-\varepsilon)^2 x}{2 (t - \varepsilon) \varepsilon t} + \frac{\sqrt{xz}}{t} \frac{I_{\nu+1}\left(\frac{\sqrt{xz}}{t} \right)}{I_{\nu}\left(\frac{\sqrt{xz}}{t} \right)} \right) \\
& = x + 2 \varepsilon \left(1 + \nu - \frac{x}{t} + \frac{\sqrt{xz}}{t} \frac{I_{\nu+1}\left(\frac{\sqrt{xz}}{t} \right)}{I_{\nu}\left(\frac{\sqrt{xz}}{t} \right)}\right) + \mathcal{O}(\varepsilon^2),
\end{aligned}
\end{equation*}
which implies that
\begin{equation*}
\lim_{\varepsilon \to 0} \frac{1}{\varepsilon} \mathbb{E}(X_{t - \varepsilon} - X_t \,|\, X_0 = z, X_t  = x)
= 2 \left(1 + \nu - \frac{x}{t} + \frac{\sqrt{xz}}{t} \frac{I_{\nu+1}\left(\frac{\sqrt{xz}}{t} \right)}{I_{\nu}\left(\frac{\sqrt{xz}}{t} \right)}\right) .
\end{equation*}
We obtain Tweedie's formula for squared Bessel process:
\begin{equation}
\label{eq:TweedieBES}
\begin{aligned}
s_\nu^{\mathrm{BESQ}}(t,x):
& = \nabla \log p_\nu^{\mathrm{BESQ}}(t,x) \\
&= \frac{\nu}{x} -\frac{1}{2t} + \frac{1}{2t \sqrt{x}}\mathbb{E}\left(\frac{\sqrt{X_0} I_{\nu+1}\left(\frac{\sqrt{xX_0}}{t}\right)}{I_{\nu}\left(\frac{\sqrt{xX_0}}{t}\right)} \,\Bigg|\, X_t = x\right).
\end{aligned}
\end{equation}

\quad More generally, we take $b(t,x) = \mu$ and $\sigma(t,x) = \sigma \sqrt{x}$,
with $\sigma > 0$ and $\mu > \frac{\sigma^2}{2}$.
The forward process evolves as
\begin{equation*}
\mathrm{d}X_t = \mu \,\mathrm{d}t + \sigma \sqrt{X_t} \,\mathrm{d}W_t, \quad X_0 \sim p_{\mathrm{data}}(\cdot).
\end{equation*}
It is easy to see that $\{\frac{4}{\sigma^2}X_t\}_{t \ge 0}$ is a squared Bessel process with index
$\nu = \frac{2 \mu}{\sigma^2} - 1$.
Hence, Tweedie's formula is
\begin{equation*}
\nabla \log p_{\mu, \sigma}^{\mathrm{BESQ}}(t,x) =\frac{\sigma^2}{4} s_{\frac{2 \mu}{\sigma^2} - 1}^{\mathrm{BESQ}}\left(t,\frac{\sigma^2 x}{4}\right),
\end{equation*}
where $s_{\nu}^{\mathrm{BESQ}}(\cdot, \cdot)$ is defined by \eqref{eq:TweedieBES}.

\subsection{CIR Processes}
We now consider the forward process of the form
\begin{equation}\label{CIR}
\mathrm{d}X_t = \alpha(t) (\mu(t) - X_t) \,\mathrm{d}t +  \sigma(t) \sqrt{X_t} \,\mathrm{d}W_t, \quad X_0 \sim p_{\mathrm{data}}(\cdot).
\end{equation}
This process is known as the CIR process \cite{CIR85}, which has been used  to model interest rate \cite{rogers1995model,shirakawa2002squared},
as well as   continuous-state branching processes with immigration
as the scaling limit of the Galton–Watson counterparts \cite{KW71, Lamp67}.
\subsubsection{Time-Independent Case}
We first consider \eqref{CIR} where $\alpha(t), \mu(t)$ and $\sigma(t)$ are time independent; i.e.
$b(t,x) = \alpha (\mu - x)$ and $\sigma(t,x) = \sigma \sqrt{x}$,
with constants $\alpha > 0$, $\sigma > 0$, and $\mu > \frac{\sigma^2}{2 \alpha}$. 
It follows  from \cite[Equation (18)]{CIR85} (see also \cite[Proposition 6.3.1.1]{JYC09})
that $\{X_t\}_{t \ge 0}$ is a space-time-changed squared Bessel process:
\begin{equation*}
X_t \stackrel{d}{=} e^{-\alpha t} Z_{\frac{\sigma^2}{4 \alpha}(e^{\alpha t} -1)},
\end{equation*}
where $\{Z_t\}_{t \ge 0}$ is squared Bessel process with index $\nu = \frac{2 \alpha \mu}{\sigma^2} - 1$.
As a result, Tweedie's formula for CIR processes is
\begin{equation}\label{Tweedie_CIR}
\nabla \log p_{\alpha, \mu, \sigma}^{\mathrm{CIR}}(t,x) = e^{\alpha t} s_{\frac{2 \alpha \mu}{\sigma^2} - 1}^{\mathrm{BESQ}}\left(\frac{\sigma^2}{4 \alpha}(e^{\alpha t} -1),  e^{\alpha t} x\right),
\end{equation}
where $s_{\nu}^{\mathrm{BESQ}}(\cdot, \cdot)$ is defined by \eqref{eq:TweedieBES}.

\subsubsection{Time-Dependent Case}
 We now turn to the general time-dependent CIR \eqref{CIR}. We choose $\alpha(t)>0$, $\mu(t)>0$, and $\sigma(t)>0$ such that $\frac{2 \alpha(t) \mu(t)}{\sigma^2(t)} - 1\equiv \nu$. By \cite[Corollary 3.1]{shirakawa2002squared}, we have:
\begin{equation*}
X_t \stackrel{d}{=} e^{-\int_0^t\alpha(s)\,\mathrm{d}s} Z_{\frac{1}{4}\int_0^t \sigma^2(s)e^{\int_0^s \alpha(u)\,\mathrm{d}u}\,\mathrm{d}s},
\end{equation*}
where $\{Z_t\}_{t \ge 0}$ is squared Bessel process with index $\nu$.\footnote{As shown in \cite{shirakawa2002squared}, the time-dependent CIR process, after space-time scaling, is squared Bessel process with a ``time-dependent" index $\nu(t):=\frac{2 \alpha(t)\mu(t)}{\sigma(t)}-1$, being characterized by the Laplace transform.
For our generation purpose, we set this index to be constant/time-independent in order to apply the result from the squared Bessel process.} Tweedie's formula is then given by
\begin{equation}\label{Tweedie_CIR2}
\nabla \log p_{\alpha, \mu, \sigma}^{\mathrm{CIR}}(t,x) = e^{\int_0^t\alpha(s)\,\mathrm{d}s} s_{\nu}^{\mathrm{BESQ}}\left(\frac{1}{4}\int_0^t \sigma^2(s)e^{\int_0^s \alpha(u)\,\mathrm{d}u}\,\mathrm{d}s,  e^{\int_0^t\alpha(s)\,\mathrm{d}s} x\right),
\end{equation}
where $s_{\nu}^{\mathrm{BESQ}}(\cdot, \cdot)$ is defined by \eqref{eq:TweedieBES}.
\subsection{(Time-Dependent) CEV Processes}

We consider $b(t,x) = \mu(t) x$ and $\sigma(t,x) = \sigma(t) x^\beta$, with $\mu(t) > 0$, $\sigma(t) > 0$ and $\beta > 1$.
The forward process is\footnote{When $\mu(t) \equiv \mu$ and $\sigma(t) \equiv \sigma$ are time-independent,
this process is known as the constant elasticity of variance (CEV) process \cite{Cox75, DS02, EM82}.}
\begin{equation}
\mathrm{d}X_t = \mu(t) X_t \,\mathrm{d}t +  \sigma(t) X_t^\beta \,\mathrm{d}W_t, \quad X_0 \sim p_{\mathrm{data}}(\cdot).
\end{equation}
It follows from \cite[Lemma 6.4.3.1]{JYC09} that $\{X_t\}_{t\ge 0}$
is a time-change of a power of squared Bessel process
\begin{equation*}
X_t \stackrel{d}{=} e^{\int_0^t\mu(s)\,\mathrm{d}s} Z_{(\beta-1)^2\int_0^t \sigma^2(s)e^{2(\beta-1)\int_0^s \mu(u)\,\mathrm{d}u}\,\mathrm{d}s}^{-\frac{1}{2(\beta-1)}},
\end{equation*}
where $\{Z_t\}_{t \ge 0}$ is squared Bessel  with index $\frac{1}{2(\beta-1)}$.\footnote{In contrast with the CIR processes, the index does not depend on the time-dependent coefficients $\mu(t)$ and $\sigma(t)$.}
As a result, Tweedie's formula for (time-dependent) CEV processes is given by
\begin{equation*}
\begin{aligned}
 \nabla \log {}&p_{\mu, \sigma, \beta}^{\mathrm{CEV}}(t,x) = \frac{-2 \beta + 1}{x}\\
 &-\frac{2(\beta - 1) e^{2 (\beta- 1)\int_0^t \mu(s)\,\mathrm{d}s}}{x^{2 \beta - 1}}s^{\mathrm{BESQ}}_{\frac{1}{2(\beta - 1)}}\left((\beta-1)^2\int_0^t \sigma^2(s)e^{2(\beta-1)\int_0^s \mu(u)\,\mathrm{d}u}\,\mathrm{d}s, e^{2 (\beta- 1)\int_0^t\mu(s)\,\mathrm{d}s} x^{-2 (\beta - 1)}\right),
\end{aligned}
\end{equation*}
where $s_{\nu}^{\mathrm{BESQ}}(\cdot, \cdot)$ is defined by \eqref{eq:TweedieBES}.

\subsection{Bessel Processes}

Here we take $b(t,x) = \frac{2 \nu + 1}{2x}$ and $\sigma(t,x) = 1$, with $\nu > 0$.
The forward process is governed by
\begin{equation}
\mathrm{d}X_t = \frac{2 \nu + 1}{2X_t}\,\mathrm{d}t +  \,\mathrm{d}W_t, \quad X_0 \sim p_{\mathrm{data}}(\cdot).
\end{equation}
By \cite[(2.i')]{PY82},
the transition density of the process $\{X_t\}_{t \ge 0}$ is
\begin{equation*}
\tilde{q}(t,x,y):=\mathrm{d}\mathbb{P}(X_{t_0 + t} \in \mathrm{d}y\,|\, X_{t_0} = x)/\mathrm{d}y = \frac{1}{t} \left( \frac{y}{x}\right)^{\nu} y \exp\left(-\frac{x^2+y^2}{2t} \right)
I_\nu\left(\frac{xy}{t} \right).
\end{equation*}
So the probability density of $(X_{t - \varepsilon}\,|\, X_0 = z, X_t = x)$ is
\begin{equation*}
\begin{aligned}
\mathrm{d} \mathbb{P}(X_{t - \varepsilon} \in \mathrm{d}y \,|\, X_0 = z, X_t = x)/\mathrm{d}y
 =\frac{y \exp\left(-\frac{ty^2}{2 (t - \varepsilon) \varepsilon} \right)I_{\nu}\left( \frac{zy}{t - \varepsilon}\right)
I_{\nu}\left( \frac{xy}{\varepsilon}\right)}{\int_0^\infty y \exp\left(-\frac{ty^2}{2 (t - \varepsilon) \varepsilon} \right)I_{\nu}\left( \frac{zy}{t - \varepsilon}\right)
I_{\nu}\left( \frac{xy}{\varepsilon}\right) \,\mathrm{d}y}.
\end{aligned}
\end{equation*}
Note from \eqref{eq:FGdef} that
$F(a,b,c; \nu) = 2 \int_0^\infty u e^{-c u^2} I_\nu(au) I_\nu(bu) \,\mathrm{d}u$.
Furthermore, let
\begin{equation*}
H(a,b,c; \nu):= \int_0^\infty u^2 e^{-cu^2} I_\nu(au) I_\nu(bu) \,\mathrm{d}u.
\end{equation*}
Then we have
\begin{equation}
\label{eq:teps2}
\mathbb{E}(X_{t - \varepsilon}\,|\, X_0 = z, X_t = x)
= \frac{2 H \left(\frac{z}{t-\varepsilon}, \frac{x}{\varepsilon}, \frac{t}{2(t - \varepsilon) \varepsilon}; \nu \right)}{F \left(\frac{z}{t-\varepsilon}, \frac{x}{\varepsilon}, \frac{t}{2(t - \varepsilon) \varepsilon}; \nu \right)}.
\end{equation}
Recall the expression of $F(a,b,c)$ from \eqref{eq:Ffor}. However,
for general $\nu$, there is no closed form expression for $H(a,b,c)$.

\quad Next, let us consider the special case $\nu = \frac{1}{2}$,
which corresponds to the three dimensional Bessel process.
Noting $I_{\frac{1}{2}}(u) = \sqrt{\frac{2}{\pi}} \frac{\sinh u}{\sqrt{u}}$, we have
\begin{equation}
\label{eq:Hfor}
\begin{aligned}
& H\left(a,b,c; \frac{1}{2}\right) \\
&\quad = \frac{2}{\pi \sqrt{ab}} \int_0^\infty u e^{-cu^2} \sinh (au) \sinh(bu) \,\mathrm{d}u \\
& \quad = \frac{1}{\pi \sqrt{ab}} \left( \int_0^\infty u e^{-cu^2} \cosh((a+b)u) \,\mathrm{d}u -\int_0^\infty u e^{-cu^2} \cosh((a-b)u) \,\mathrm{d}u \right) \\
& \quad = \frac{1}{\pi \sqrt{ab}} \left(\frac{a+b}{2c} \int_0^\infty e^{-cu^2} \sinh((a+b)u) \,\mathrm{d}u -  \frac{a-b}{2c} \int_0^\infty e^{-cu^2} \sinh((a-b)u) \,\mathrm{d}u\right) \\
& \quad = \frac{(a+b)\exp\left(\frac{(a+b)^2}{4c} \right) \erf\left(\frac{a+b}{2 \sqrt{c}} \right) - (a-b)\exp\left(\frac{(a-b)^2}{4c} \right) \erf\left(\frac{a-b}{2 \sqrt{c}} \right)}{4 \sqrt{\pi abc^3}},
\end{aligned}
\end{equation}
where we use the formula \cite[(3.546.1)]{GR15} in the last equation,
with $\erf(u):= \frac{2}{\sqrt{\pi}} \int_{0}^u e^{-z^2}\,\mathrm{d}z$ the error function of the standard normal.
Combining \eqref{eq:teps2}, \eqref{eq:Ffor} and \eqref{eq:Hfor} yields:
\begin{equation*}
\begin{aligned}
\mathbb{E}(X_{t - \varepsilon}\,|\, X_0 = z, X_t = x)
&= \frac{\varepsilon z + (t-\varepsilon) x}{2t} \frac{e^{\frac{xz}{t}}}{\sinh \left( \frac{xz}{t}\right)} \erf\left( \frac{\varepsilon z + (t - \varepsilon) x}{\sqrt{2 t (t- \varepsilon) \varepsilon}}\right) \\
& \qquad \qquad   - \frac{\varepsilon z - (t-\varepsilon) x}{2t}\frac{e^{-\frac{xz}{t}}}{\sinh\left(\frac{xz}{t}\right)} \erf\left( \frac{\varepsilon z - (t - \varepsilon) x}{\sqrt{2 t (t- \varepsilon) \varepsilon}}\right) \\
& = x + \frac{\varepsilon(z-x)}{t} \coth\left(\frac{xz}{t} \right) + \mathcal{O}(\varepsilon^2),
\end{aligned}
\end{equation*}
which implies 
\begin{equation*}
\lim_{\varepsilon \to 0} \frac{1}{\varepsilon} \mathbb{E}(X_{t - \varepsilon} - X_t \,|\, X_0 = z, X_t  = x)
= \frac{z-x}{t} \coth\left( \frac{xz}{t}\right).
\end{equation*}
This leads to Tweedie's formula for three-dimensional Bessel process:
\begin{equation}
\label{eq:sbes}
s^{\mathrm{BES}}_{\frac{1}{2}}(t,x):=\nabla \log p^{\mathrm{BES}}_{\frac{1}{2}}(t,x) = \frac{1}{x} + \frac{1}{t} \mathbb{E}\left( (X_0 - x) \coth\left(\frac{x X_0}{t} \right) \Bigg| X_t = x \right).
\end{equation}

\quad A slightly more general case than the above is to choose  $b(t,x) = \frac{\sigma^2}{x}$ and $\sigma(t,x) = \sigma$, with $\sigma > 0$.
The forward process is
\begin{equation*}
\mathrm{d}X_t = \frac{\sigma^2}{X_t}\,\mathrm{d}t +  \sigma \,\mathrm{d}W_t, \quad X_0 \sim p_{\mathrm{data}}(\cdot).
\end{equation*}
It is easy to see that $\{\frac{1}{\sigma} X_t\}_{t \ge 0}$ is a three-dimensional Bessel process. As a result, Tweedie's formula is
\begin{equation*}
\nabla \log p^{\mathrm{BES}}_{\frac{1}{2},\sigma}(t,x) = \sigma s^{\mathrm{BES}}_{\frac{1}{2}}(t,\sigma x),
\end{equation*}
where $s^{\mathrm{BES}}_{\frac{1}{2}}(\cdot,\cdot)$ is defined by \eqref{eq:sbes}.

\section{Numerical experiments}
\label{sc4}
\subsection{Diffusion Models}\label{sc4-1}

\quad This section provides numerical experiments using non-Gaussian models, where we focus on GBM and CIR processes, to formulate the forward process of SDE-based diffusion models. The experimental details are deferred to Appendix~\ref{app:experiment}. For the sake of comparison, we also presents the results generated by Gaussian models. We first derive denoising score matching objectives for these models right from the Tweedie's formula established.

\medskip
{\bf Variance Exploding Processes.} From \eqref{Tweedie_BM} we know
$$
\mathbb{E}(X_0\mid X_t)=X_t+\Sigma^2(t)\cdot \nabla \log p_{\sigma}^{\mathrm{VE}}(t,X_t),
$$
where we have denoted $\Sigma^2(t):=\int_0^t \sigma^2(s)\,\mathrm{d}s$. Thus, the denoising score matching objective is
\begin{align*}
&\min_\theta \mathbb{E} \left| X_t+\Sigma^2(t)\cdot s_\theta^{\mathrm{VE}}(t,X_t)-X_0 \right|^2\\
={}&\min_\theta \mathbb{E}_{t\sim \mathcal{U}[0,T], \, X_0\sim p_\mathrm{data}(\cdot),\, Z\sim \mathcal{N}(0,1)} \Sigma^2(t) | \Sigma(t)\cdot s_\theta^{\mathrm{VE}}(t,X_0+\Sigma(t) Z)+Z |^2,
\end{align*}
where the equality is because $X_t=X_0+\Sigma(t)  Z$ with $Z\sim \mathcal{N}(0,1)$. One can alternatively adopt the reparameterization $\epsilon_\theta^{\mathrm{VE}}(t,x):=-\Sigma(t)\cdot s_\theta^\mathrm{VE}(t,x)$ to obtain a noise-prediction objective
\begin{equation*}
\min_\theta \mathbb{E}_{t\sim \mathcal{U}[0,T], \, X_0\sim p_\mathrm{data}(\cdot),\, Z\sim \mathcal{N}(0,1)} \Sigma^2(t) | \epsilon_\theta^{\mathrm{VE}}(t,X_0+\Sigma(t) Z)-Z |^2.
\end{equation*}

\medskip
{\bf Variance Preserving Processes.}  By~\eqref{Tweedie_VP}, we have
$$
\mathbb{E}(X_0\mid X_t)=\frac{1}{A(t)}(X_t+\Sigma^2(t)\cdot \nabla \log p_{\alpha}^{\mathrm{VP}}(t,X_t)),
$$
where $A(t):=\exp(-\int_0^t \alpha(s)\,\mathrm{d}s)$ and $ \Sigma^2(t):=1-A^2(t)=1-\exp(-2\int_0^t \alpha(s)\,\mathrm{d}s)$. Hence, the corresponding denoising score matching objective is given by
\begin{align*}
&\min_\theta \mathbb{E} \left| \frac{1}{A(t)}(X_t+\Sigma^2(t)\cdot s_\theta^{\mathrm{VP}}(t,X_t))-X_0 \right|^2\\
={}&\min_\theta \mathbb{E}_{t\sim \mathcal{U}[0,T], \, X_0\sim p_\mathrm{data}(\cdot),\, Z\sim \mathcal{N}(0,1)} \frac{\Sigma^2(t)}{A^2(t)} | \Sigma(t)\cdot s_\theta^{\mathrm{VP}}(t,A(t) X_0+\Sigma(t)  Z)+Z |^2\\
={}&\min_\theta \mathbb{E}_{t\sim \mathcal{U}[0,T], \, X_0\sim p_\mathrm{data}(\cdot),\, Z\sim \mathcal{N}(0,1)} \frac{1}{\lambda_t} | \epsilon_\theta^{\mathrm{VP}}(t,A(t) X_0+\Sigma(t)  Z)-Z |^2,
\end{align*}
where the first equality is because $
X_t=A(t) X_0+\Sigma(t) Z, \ Z\sim \mathcal{N}(0,1)$, and in the second inequality we let $\epsilon_\theta^\mathrm{VP}(t,x):=-\Sigma(t)\cdot s_\theta^\mathrm{VP}(t,x)$ be the noise-prediction reparameterization and $\lambda_t:=A^2(t)/\Sigma^2(t)$ the signal-to-noise ratio.

\medskip
{\bf Geometric Brownian Motion.}
It follows from \eqref{Tweedie_GBM} that
$$
\mathbb{E} (\log X_0\,|\,X_t)=\Sigma^2(t) X_t \nabla \log p_{\mu, \sigma}^{\mathrm{GBM}}(t,X_t)-\Sigma^2(t)  \left(\frac{U(t)}{\Sigma^2(t)}-\frac{3}{2}\right)+\log X_t,
$$
where $U(t):=\int_0^t \mu(s)\,\mathrm{d}s$ and $\Sigma^2(t):=\int_0^t \sigma^2(s)\,\mathrm{d}s$.
Replacing the true score $\nabla \log p_{\mu, \sigma}^{\mathrm{GBM}}(t,x)$ with parameterized family $\{s_\theta^{\mathrm{GBM}}(t,x)\}_\theta$,
we get the denoising score matching objective:
\begin{align*}
&\min_\theta \mathbb{E} \left| \Sigma^2(t)  X_t s_\theta^{\mathrm{GBM}}(t,X_t)-\Sigma^2(t)  \left(\frac{U(t)}{\Sigma^2(t)}-\frac{3}{2}\right)+\log X_t-\log X_0 \right|^2\\
={}&\min_\theta \mathbb{E}_{t\sim \mathcal{U}[0,T], \, X_0\sim p_\mathrm{data}(\cdot),\, Z\sim \mathcal{N}(0,1)} \Sigma^2(t) | \Sigma(t)(1+X_t s_\theta^{\mathrm{GBM}}(t,X_t))+Z |^2,
\end{align*}
where $X_t=X_0\exp(U(t)-\frac{1}{2}\Sigma^2(t)+\Sigma(t) Z)$
with $Z \sim \mathcal{N}(0,1)$.

\quad We need to be mindful  that na\"ively applying the above score matching objective for training leads to an excessively large initial loss due to the multiplicative exponential Gaussian noise in the forward process. Moreover, even when the score network is well trained, sampling from the reverse-time dynamics may suffer from numerical blow-up because of the state-dependent diffusion coefficient. To improve both the training and sampling stability, we choose a noise-prediction reparameterization:
\begin{equation*}
\epsilon_\theta^{\mathrm{GBM}}(t,x):=-\Sigma(t)(1+x s_\theta^{\mathrm{GBM}}(t,x)),
\end{equation*}
under which the score-matching objective can be equivalently rewritten as:
\begin{equation}
\min_\theta \mathbb{E}_{t\sim \mathcal{U}[0,T]} \left(\Sigma^2(t)  \ \mathbb{E}_{X_0\sim p_\mathrm{data}(\cdot), \, Z\sim \mathcal{N}(0,1)} \left| \epsilon_\theta^{\mathrm{GBM}}\left(t,X_0\exp\left(U(t)-\frac{1}{2}\Sigma^2(t)+\Sigma(t) Z\right)\right)-Z \right|^2\right).
\end{equation}
Then the Euler--Maruyama sampler for the backward generation is given by
\begin{align*}
y_{t+\Delta t}={}&y_t + ((2\sigma^2(T-t)-\mu(T-t))y_t+\sigma^2(T-t) y_t^2 s_\theta^{\mathrm{GBM}}(T-t, y_t))\Delta t + \sigma(T-t) y_t\sqrt{\Delta t}z_t\notag\\
={}&y_t\left(1+\left(\sigma^2(T-t)-\mu(T-t)-\frac{\sigma^2(T-t)}{\sigma_{T-t}}\epsilon_\theta^{\mathrm{GBM}}(T-t,y_t)\right)\Delta t + \sigma(T-t) \sqrt{\Delta t}z_t \right), \label{EMsampler_GBM_MNIST}
\end{align*}
where $z_t\sim \mathcal{N}(0,1)$ and $y_0\sim p_{\tiny \mbox{noise}}^{\tiny \mbox{GBM}}(\cdot)=\mathrm{LogNormal}(U(T)-\frac{1}{2}\Sigma^2(T), \Sigma^2(T))$.

\medskip
{\bf CIR Processes.}
We set $\sigma(t) = \sqrt{2\alpha(t)}$ and $\mu(t)\equiv \mu$, so that the stationary distribution is $p_{\tiny \mbox{noise}}^{\tiny \mbox{CIR}}(\cdot) = \mathcal{G}(\mu, 1)$, a Gamma distribution.
Let $A(t):=\int_0^t\alpha(s)\,\mathrm{d}s$. The conditional distribution
$X_t\,|\,X_0$ is $\frac{1-e^{-A(t)}}{2}K$, where $K\sim\chi^2(2\mu,2X_0\frac{1}{e^{A(t)}-1});$
see, e.g., \cite{richemond2022categorical}. Here, $\chi^2(k,m)$ is the non-central chi-squared distribution with $k$ degrees of freedom and non-centrality parameter $m$. The denoising score-matching objective can be derived from \eqref{Tweedie_CIR} as follows:
\begin{align}
\min_\theta \mathbb{E}\ \lambda(t) \left| \left(e^{-\frac{1}{2}{A(t)}}(e^{A(t)}-1)\left(s_\theta^\mathrm{CIR}(t,X_t)+\frac{1-\mu}{X_t}\right)+e^{\frac{1}{2}A(t)} \right)\sqrt{X_t}-\frac{\sqrt{X_0}I_\mu\left(\frac{2\sqrt{X_t X_0}}{e^{-\frac{1}{2}{A(t)}}(e^{A(t)}-1)}\right)}{I_{\mu-1}\left(\frac{2\sqrt{X_t X_0}}{e^{-\frac{1}{2}{A(t)}}(e^{A(t)}-1)}\right)} \right|^2,
\end{align}
where $\lambda(t)$ is the weighting function and the expectation is taken over $t\sim \mathcal{U}[0,T],\ X_0\sim p_\mathrm{data}(\cdot), \ K\sim\chi^2(2\mu,2X_0\frac{1}{e^{A(t)}-1})$, and $X_t=\frac{1-e^{-A(t)}}{2}K$. The Euler–Maruyama sampler for the backward generation is given by
\begin{equation}\label{EM_CIR}
y_{t+\Delta t}= y_t+\alpha(T-t)(2y_t s_\theta^\mathrm{CIR}(T-t, y_t)+2-\mu+y_t)\Delta t+\sqrt{2\alpha(T-t) y_t\Delta t}z_t,
\end{equation}
where $ z_t\sim \mathcal{N}(0,1)$ and $y_0\sim p_{\tiny \mbox{noise}}^{\tiny \mbox{CIR}}(\cdot)$. In our experiments below, we will set $\mu=1$, 
$\lambda(t)=e^{-A(t)}$,  and adopt the reparameterization $\epsilon_\theta^\mathrm{CIR}(t,x):=(1-e^{-A(t)})\cdot s_\theta^\mathrm{CIR}(t,x)+1$ to improve training efficiency and stability.

\subsubsection{Image Generation}\label{sec:mnist}

We present experimental results on the MNIST dataset. A key property of GBM and CIR processes is that, starting from a non-negative initial condition, the process remains non-negative for all time. Motivated by this property, we consider simplex diffusion models for categorical data \cite{floto2023diffusion, richemond2022categorical} which perform the diffusion process in the space of probability simplex, leveraging the discrete nature of the data.

\quad Following \cite{floto2023diffusion}, we create a discrete version of the MNIST dataset by mapping each pixel value in $\{0,1,\ldots,255\}$ to three categories: dark (0–85), medium (86–170), and bright (171–255). Each pixel is then represented as a 3-dimensional real vector. Specifically, we encode a dark/medium/bright pixel as $(a+b, a, a)$, $(a, a+b, a)$, or $(a, a, a+b)$ in $\mathbb{R}_+^3$, respectively, for some $a,b\geq 0$, thereby constructing the training dataset consisting of ``dimension-reduced" MNIST images (see Figure~\ref{fig:real_MNIST}). For generated samples, the pixel category is determined by the index of the largest component of the resulting 3-dimensional vector: if the maximum occurs in the first, second, or third dimension, the pixel is classified as dark, medium, or bright, respectively.

\begin{figure}[!htbp]
    \centering
    \caption{Real and Generated MNIST Images}
    \label{fig:generated_MNIST}

    \begin{subfigure}{0.45\textwidth}
        \centering
        \includegraphics[width=0.85\textwidth]{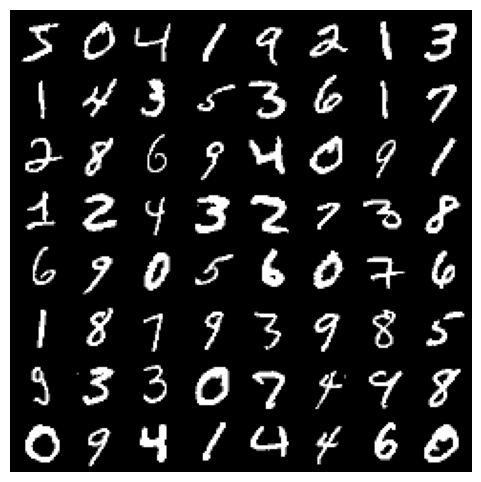}
        \caption{Preprocessed Dataset}
        \label{fig:real_MNIST}
    \end{subfigure}

    \vspace{1em}

        \begin{subfigure}{0.45\textwidth}
        \centering
        \includegraphics[width=0.85\textwidth]{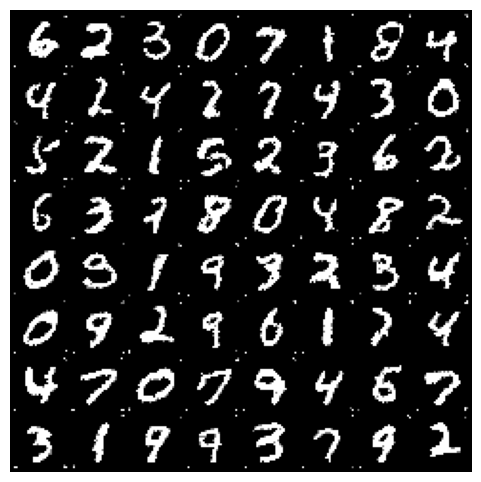}
        \caption{GBM-Based Samples}
        \label{fig:generated_MNIST_GBM}
    \end{subfigure}
  \hspace{0.01\textwidth}
        \begin{subfigure}{0.45\textwidth}
        \centering
        \includegraphics[width=0.85\textwidth]{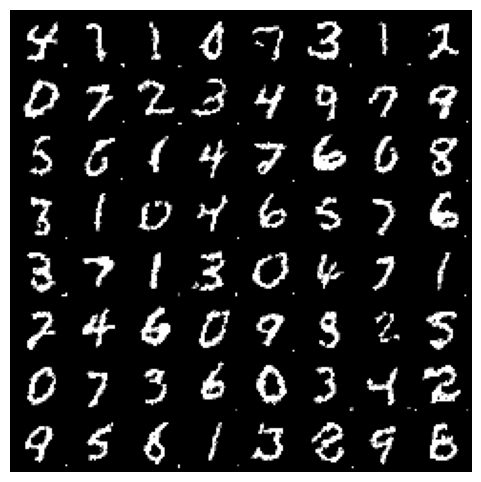}
        \caption{VE-Based Samples}
        \label{fig:generated_MNIST_BM}
    \end{subfigure}

        \vspace{1em}
    \begin{subfigure}{0.45\textwidth}
        \centering
        \includegraphics[width=0.85\textwidth]{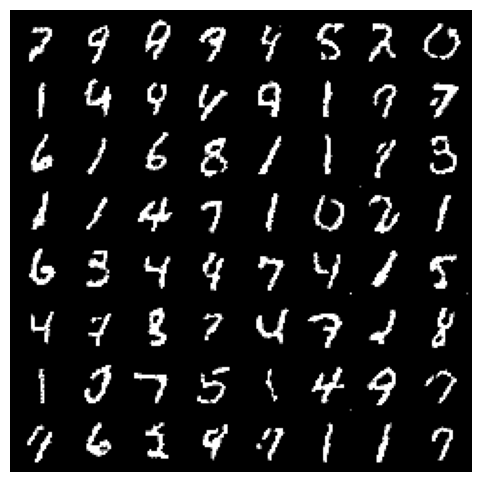}
        \caption{CIR-Based Samples}
        \label{fig:generated_MNIST_CIR}
    \end{subfigure}
      \hspace{0.01\textwidth}
        \begin{subfigure}{0.45\textwidth}
        \centering
        \includegraphics[width=0.85\textwidth]{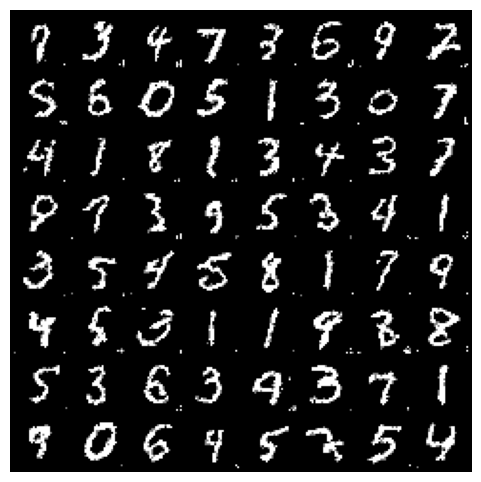}
        \caption{VP-Based Samples}
        \label{fig:generated_MNIST_VP}
    \end{subfigure}
\end{figure}

\quad Since both GBM and VE exhibit exploding variance and do not admit limiting distributions, whereas both CIR and VP have limiting distributions, we conduct a fair comparison by grouping models with similar properties. Specifically, we compare GBM with VE and CIR with VP, respectively. More details of the experiment can be found in Appendix \ref{app:experiment}.

\quad The visualization results are presented in Figure~\ref{fig:generated_MNIST}. The GBM-based samples exhibit a noisier background with moderate stroke clarity, where digits are recognizable but show some irregularity in stroke width and curvature. By contrast, the VE-based samples demonstrate cleaner digit boundaries and more uniform stroke thickness across the grid, suggesting that the VE SDE provides a more stable diffusion trajectory that better preserves the structural regularity of handwritten numerals. The sampling error for GBM-based models mainly stems from numerical instability in the SDE sampler during reverse-time sampling, as well as inefficient mixing of the forward process. First, in contrast to the additive Gaussian noise in VE models, the multiplicative exponential Gaussian noise in the GBM forward process can induce large fluctuations in the sampling trajectories, leading to numerical instability. In addition, the state-dependent diffusion coefficient $\sigma(t)x$ introduces a multiplicative dependence on the previous state, $y_t$, in each iteration, which would further accumulate numerical errors. Nevertheless, these adverse effects are partially alleviated by assigning pixel categories according to the index of the maximum component. Second, the absence of a stationary distribution of GBM leads to a mismatch between the true terminal distribution of the forward process, $p(T,\cdot)$, and the initialization distribution used for reverse-time sampling, $p_{\tiny \mbox{noise}}^{\tiny \mbox{GBM}}(\cdot)$, thereby introducing additional sampling error.

\quad The CIR-based samples demonstrate a clear visual advantage over their VP-based counterparts, exhibiting sharper stroke definition, more consistent digit morphology, and a notably cleaner background across the generated grid. The VP-based samples suffer from higher background noise levels and less precise stroke boundaries, with several digits showing signs of blurring or incomplete formation. These observations suggest the effectiveness of CIR-based models. Notably, as pointed out by \cite{richemond2022categorical}, using a CIR process as the forward process for simplex diffusion models seems more natural than using VE or VP SDEs, since the normalization of the CIR limiting Gamma distribution is a Dirichlet distribution, which serves as the conjugate prior of the categorical distribution and plays a role somewhat analogous to that of the Gaussian distribution in continuous diffusion models. We note that the Euler–Maruyama sampler \eqref{EM_CIR} for the CIR process may suffer from numerical instability in the final few steps, due to the time-dependent diffusion coefficient $\sqrt{2\alpha(t)x}$, in contrast to the VP formulation. Moreover, the VP framework is generally preferred in large-scale experiments because of the simplicity of its score-matching objective, as well as the flexibility in designing loss weighting functions and reverse-time samplers. Improving the training and sampling efficiency and the stability of CIR-based models remains an important direction for future work.

\quad Table~\ref{tab:fid} reports the relative FID \cite{heusel2017gans} of the four diffusion models, computed based on 2000 generated samples for each model. Since our primary interest lies in the effectiveness of these non-Gaussian models, we focus on relative sample quality rather than absolute performance. Accordingly, we treat the VE model as the baseline by normalizing its FID score as $1.00$ and report the scores of the VP-, GBM-, and CIR-based models relative to this baseline. Lower relative FID values indicate better performance. As shown, the CIR-based model achieves the best performance in this context of simplex diffusion models and the FID results are consistent with the visual quality presented in Figure~\ref{fig:generated_MNIST}. We note that further improvements in overall sample quality would require more sophisticated network architectures and more refined training and sampling strategies for different forward SDEs, which are beyond the scope of this work. 

\begin{table}[!t]
\centering
\caption{Relative FID Scores of Different Diffusion Models}
\label{tab:fid}
\begin{tabular}{lcccc}
\toprule
Model & VE & GBM & VP & CIR \\
\midrule
Relative FID $\downarrow$
& 1.00
& 1.28
& 1.24
& 0.80\\
\bottomrule
\end{tabular}
\end{table}

\subsubsection{Financial Time Series for Portfolio Management}\label{sec:fin}

For financial time series, we consider $N=4$ stocks: AAPL, AMZN, JPM, and TSLA. Let $r_t^i$ denote the log return of stock $i$ at time $t$, $i \in [N]$. For a consecutive time window $\{t+1, \dots, t+L\}$ of length $L$, we construct a data point
$$
(r_{t+1}^1,\cdots,r_{t+L}^1;r_{t+1}^2,\cdots,r_{t+L}^2;\cdots;r_{t+1}^{N},\cdots,r_{t+L}^{N})\in \mathbb{R}^{NL}.
$$

We set $L=64$ and let $t$ range over all trading days since January 1, 2010. This yields a dataset with 3,587 data points.

\quad  We present results generated by GBM- and VP-based diffusion models for their better empirical performances and relative simplicity \footnote{We also experimented with VE- and CIR-based models. For VE-based models, although the score network can be well-trained, the explosive behavior of the VE process and the high sensitivity of financial data leads to a shifted generated distribution. For CIR-based models, the score matching objective involves modified Bessel functions and multiple exponential terms, leading to high computational complexity and numerical instability, which makes the resulting sampling process unsuitable for this financial setting.}.
To apply GBM as the forward process, we exponentiate the log returns to obtain positive price ratios:
$$
e^{r_t^i}=\frac{p_t^i}{p_{t-1}^i}>0,
$$
where $p_t^i > 0$ is the price of stock $i$ at time $t$. The resulting data points take the form:
$$
(e^{r_{t+1}^1},\cdots,e^{r_{t+L}^1};e^{r_{t+1}^2},\cdots,e^{r_{t+L}^2};\cdots;e^{r_{t+1}^{N}},\cdots,e^{r_{t+L}^{N}})\in\mathbb{R}^{NL}_+.
$$

\quad Following \cite{guo2025diffusion, GTX26}, we evaluate 64-day cumulative log-returns under three portfolio strategies: the equal-weight portfolio, the Markowitz global minimum variance portfolio (GMVP), and the risk-parity portfolio. After converting the generated price ratios to log-returns by taking logarithm, we compare the distribution of the generated samples with that of the real data in terms of mean, standard deviation, quantiles, and overall distributional shape.

\renewcommand{\arraystretch}{1.3}
\begin{table}[!htbp]
\caption{Real and Generated 64-Day Log-Return Statistics of GBM- and VP-Based Models}
\label{tab:3port}
\centering
\resizebox{\textwidth}{!}{%
\begin{tabular}{l ccc ccc ccc}
\toprule
 & \multicolumn{3}{c}{Equal-Weight}
 & \multicolumn{3}{c}{GMVP}
 & \multicolumn{3}{c}{Risk-Parity} \\
\cmidrule(lr){2-4} \cmidrule(lr){5-7} \cmidrule(lr){8-10}
Statistics
& Real & GBM & VP
& Real & GBM & VP
& Real & GBM & VP \\
\midrule
Mean
& 6.56\% & 6.52\% &6.24\% & 6.29\% & 6.51\% & 6.28\% & 6.13\% & 6.14\% & 6.00\% \\
Median
& 6.01\% & 6.54\% & 6.35\% & 6.96\% & 6.55\% & 6.40\% & 6.58\% & 6.15\% & 6.03\% \\
Std Dev
& 11.65\% & 11.41\% & 8.99\% & 8.45\% & 11.41\% & 8.99\% & 9.73\% & 10.35\% & 7.96\% \\
1\% Quantile
& -25.20\% & -19.55\% & -16.83\% & -16.91\% & -19.55\% & -16.79\% & -21.02\% & -17.99\% & -14.06\% \\
5\% Quantile
& -11.74\% & -12.09\% & -8.73\% & -9.26\% & -12.09\% & -8.68\% & -10.61\% & -10.74\%  & -7.55\%\\
10\% Quantile
& -6.31\% & -7.99\% & -5.03\% & -4.92\% & -7.99\% & -4.99\% & -5.21\% & -7.20\% & -3.95\% \\
25\% Quantile
& -0.08\% & -1.05\% & 0.61\% & 2.05\% & -1.05\% & 0.66\% & 1.07\% & -0.67\% & 1.08\% \\
\bottomrule
\end{tabular}
}
\end{table}

\begin{figure}[htbp]
    \centering
    \caption{Real and Generated 64-Day Sum Log-Returns of Three Portfolios}\label{fig:return}
    \begin{subfigure}{0.9\textwidth}
        \centering
   \includegraphics[width=\linewidth]{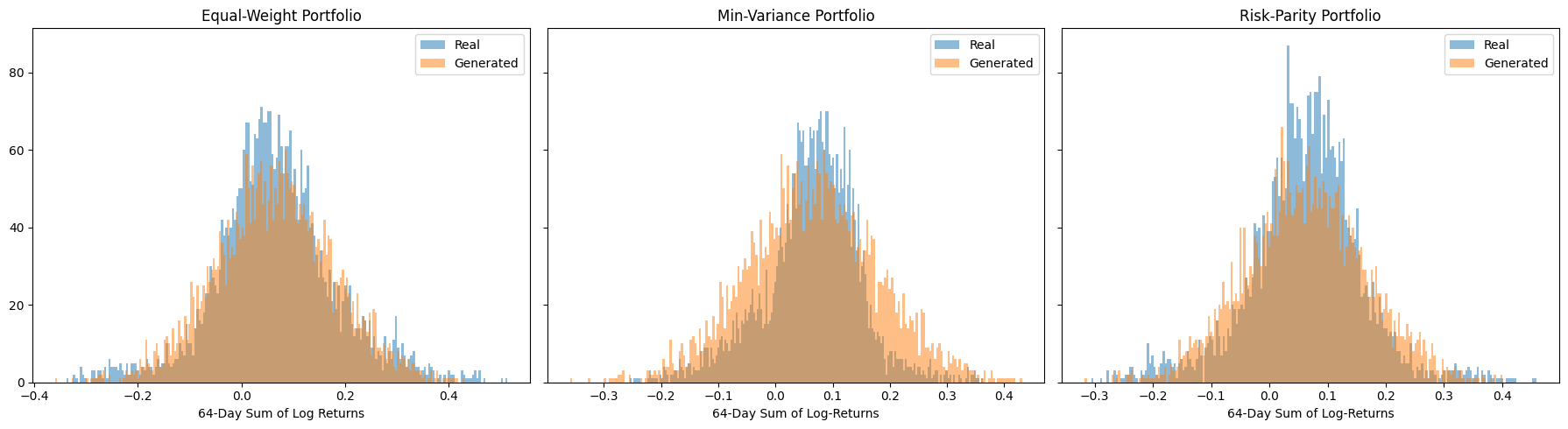}
        \caption{GBM-Based Samples}\label{fig:gbm}
    \end{subfigure}

    \vspace{0.4cm}

    \begin{subfigure}{0.9\textwidth}
        \centering
        \includegraphics[width=\linewidth]{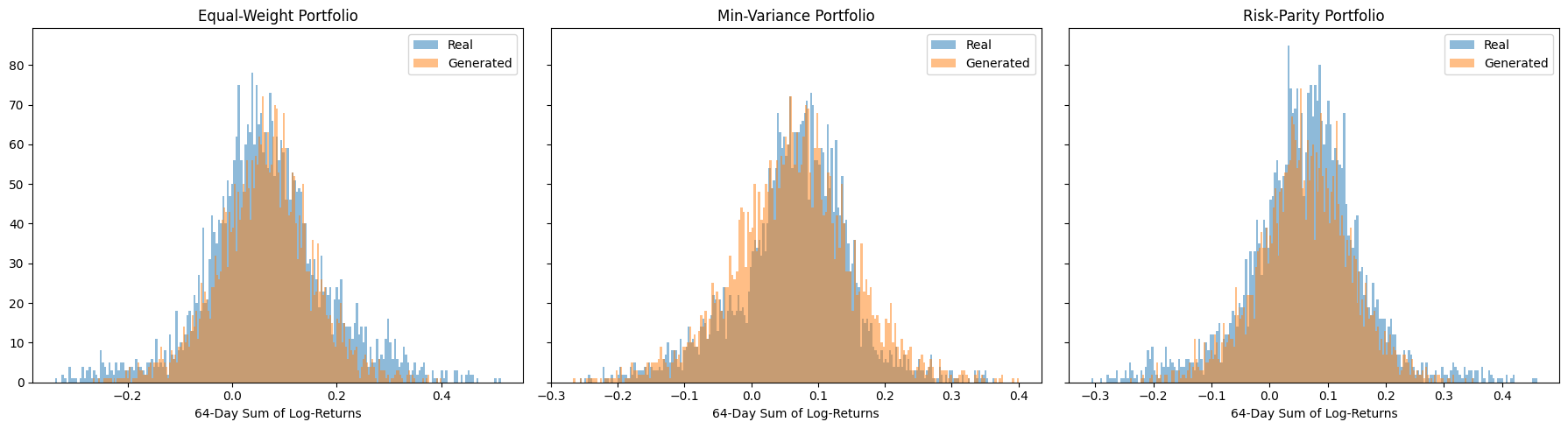}
        \caption{VP-Based Samples}\label{fig:vp}
    \end{subfigure}
\end{figure}

\quad The statistical properties of samples generated by time-dependent GBM- and VP-based models are summarized in Table~\ref{tab:3port}  and Figure~\ref{fig:return} presents the histograms of the real and generated 64-day cumulative log-returns under the three portfolio constructions.
We observe that the GBM-based model generates log-returns that align more closely with the real data in terms of mean and volatility than the VP-based model, particularly for the Equal-Weight and Risk-Parity portfolios. Both models exhibit a modest mismatch with the real data in the tail regions, which is largely attributable to the high sensitivity of financial time series to numerical errors inherent in the simulation pipeline.

\quad Overall, our results, if still preliminary,  show that the CIR-based simplex diffusion model excels in MNIST image generation while the GBM-based model outperforms in financial data generation. These findings highlight the potentials and effectiveness of non-Gaussian processes for formulating diffusion models in specific settings and tasks.

\subsection{Empirical Bayes Estimation}

Beyond its importance for diffusion models, Tweedie’s formula is also a cornerstone of empirical Bayes methods \cite{Efron11}. Here, we present an application to empirical Bayes estimation based on the BESQ version, which goes beyond Efron's generalized Tweedie's formula. In Appendix~\ref{app:gbm}, we also present an empirical Bayes estimation based directly on the GBM, and compare it with Brownian motion (BM) models by taking logarithm.

\quad Suppose there are some large number $N$ of possibly correlated noncentral chi-squared variates $z_i>0$ have been observed, each with its own unknown noncentrality parameter $u_i$,
$$
z_i\sim \chi^2(3,u_i),\quad i=1,2,\cdots,N.
$$
We aim to estimate the corresponding $u_i$ values through an empirical Bayes approach. Suppose that $u$ has been sampled from a prior distribution $g(\cdot)$, and then $z\,|\,u\sim \chi^2(3,u)$:
$$
u_1,\cdots,u_N\sim g(\cdot), \quad \text{and}\ z_i\sim \chi^2(3,u_i),\quad i=1,\cdots,N.
$$
Since $X_t\sim t\cdot \chi^2(2(\nu+1),X_0/t)$ for all $t>0$ in BESQ \eqref{eq:SDEnu}, by Tweedie's formula \eqref{eq:TweedieBES} (in which $\nu=\frac{1}{2}$, $t=1$, $X_0=u$, and $X_1=z$ are specified), we obtain the posterior expectation of $f(u,z)$ given $z$ as
$\mathbb{E}(f(u,z)\,|\, z)=(2 s(z)+1)\sqrt{z}$,
where $f(u,z):=\sqrt{u}\coth{\sqrt{uz}}$ and $s(\cdot)$ denotes the score function of the marginal density $p(\cdot)$ of $z$, i.e.,
$s(z)=\frac{\mathrm{d}}{\mathrm{d}z}\log p(z)$.
 Thus, the corresponding empirical Bayes formula is
\begin{equation}\label{EBF:BESQ}
f(\hat{u}_i,z_i)=(2\hat{s}(z_i)+1) \sqrt{z_i}\quad \text{for}\ i=1,\cdots,N,
\end{equation}
where $\hat{s}(\cdot)$ is an estimator of the true score $s(\cdot)$, which can be obtained by leveraging the observations $z_1,\cdots,z_N$ through Lindsey’s method \cite{efron2008microarrays,Efron11,efron2012large} and $\hat{u}_i$ is an estimate of $u_i$. Note that $f(u,z)$ is strictly increasing w.r.t. $u$ for any fixed $z>0$; so there exists a unique $\hat{u}_i$ satisfying \eqref{EBF:BESQ} for a given $z_i$.

\begin{figure}[!h]
    \centering
       \caption{Gamma Example}
    \label{fig:BESQ_combined}
    \begin{subfigure}{0.48\textwidth}
        \centering
        \includegraphics[width=0.9\textwidth]{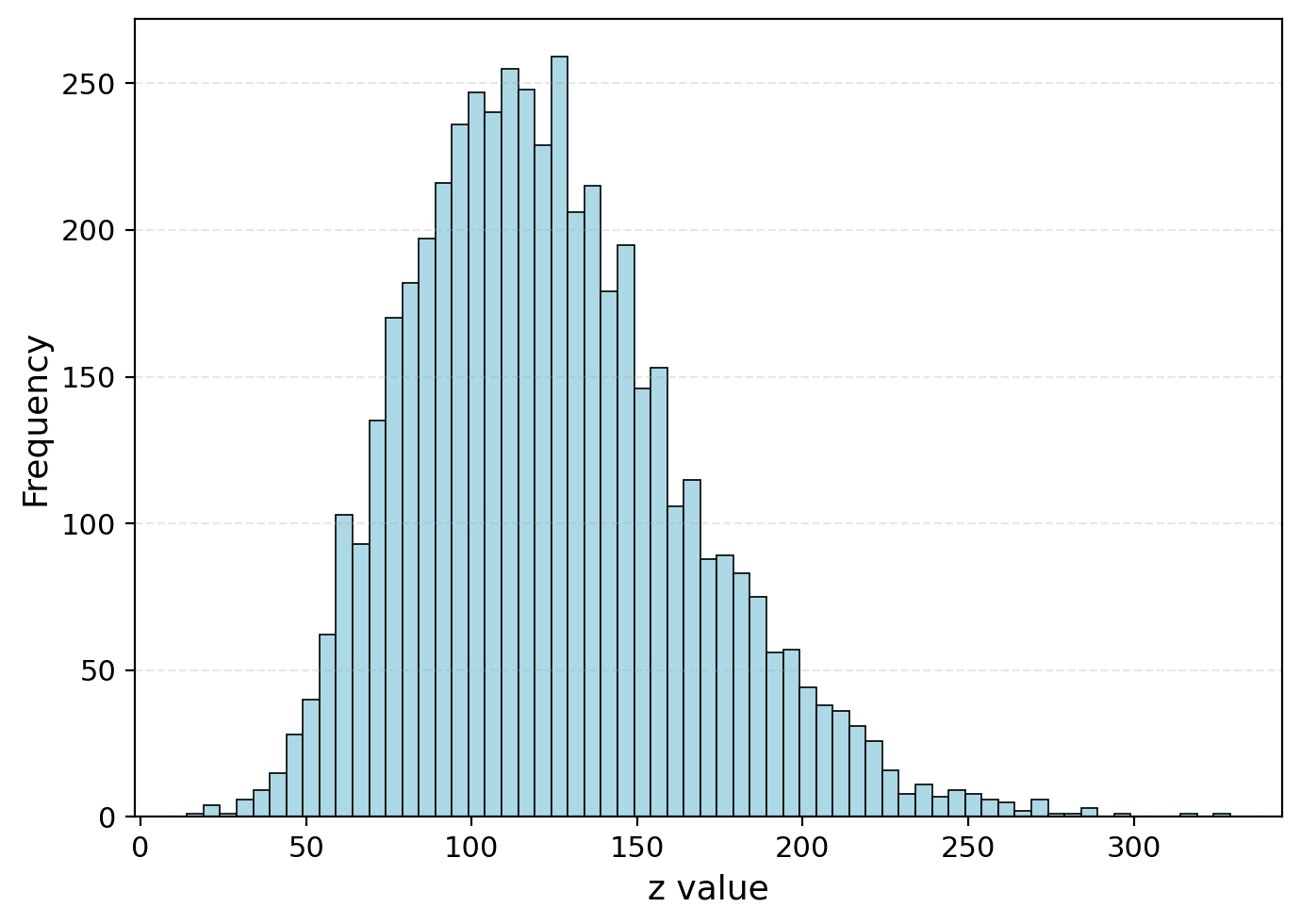}
        \caption{Histogram of $5000\ z_i$'s}
        \label{fig:frequency_BESQ}
    \end{subfigure}
    \hfill
    \begin{subfigure}{0.48\textwidth}
        \centering
        \includegraphics[width=0.9\textwidth]{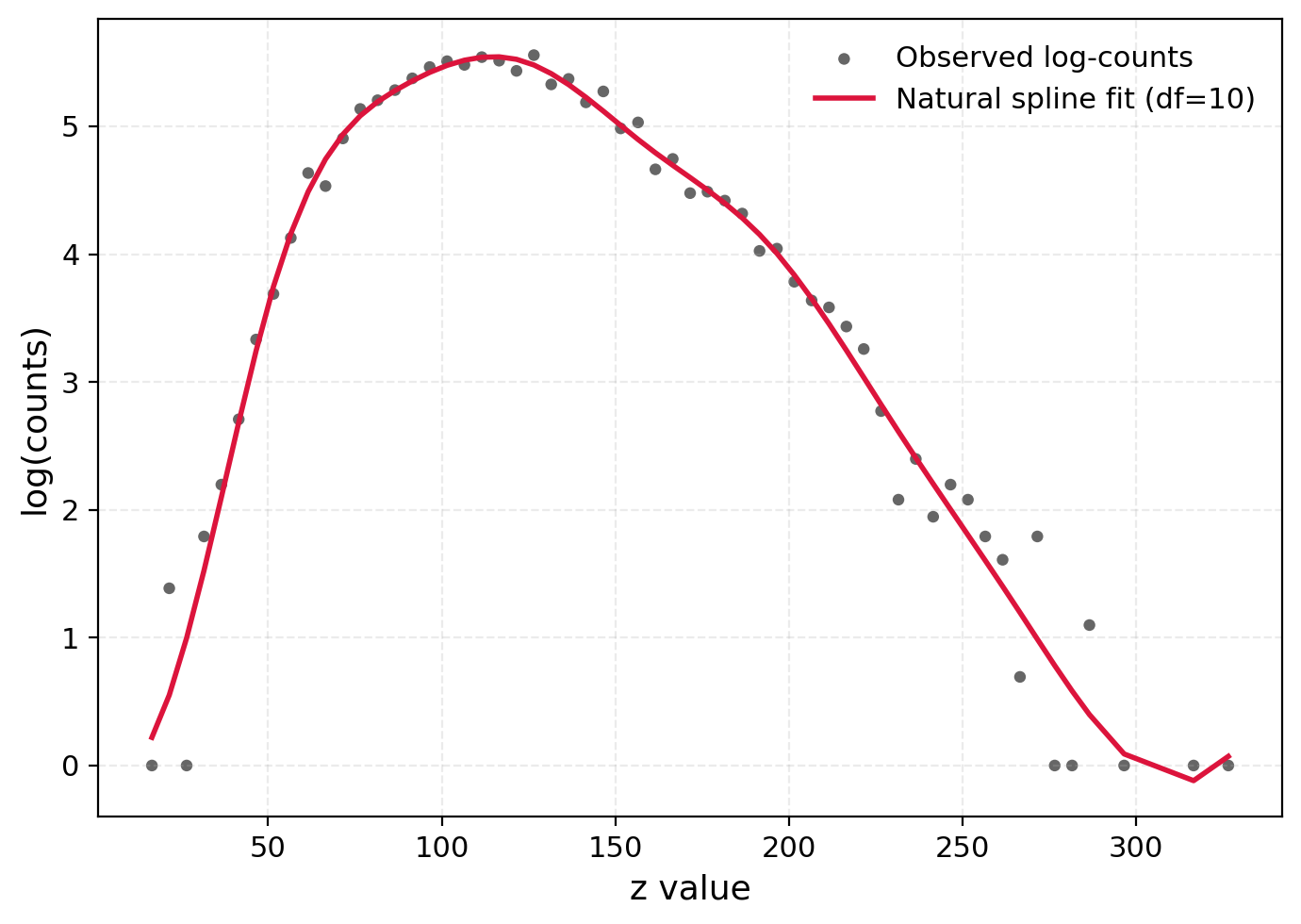}
        \caption{Natural Spline Fit with 10 Degrees of Freedom}
        \label{fig:spline_BESQ}
    \end{subfigure}
\end{figure}

\quad We randomly sample $N=5000$ values of $u_i$ from Gamma distribution $\Gamma(12,10)$.
Figure~\ref{fig:frequency_BESQ} shows the frequency of $5000$ generated $z_i$'s, where there are $63$ bins. Denote the center of the $k$-th bin by $x_k$ and the corresponding bar height by $y_k$ for $k=1,\dots,63$. Figure~\ref{fig:spline_BESQ} shows $\log y_k$ against $x_k$, with bins satisfying $y_k=0$ excluded, and a natural spline with 10 degrees of freedom is fitted to the points. The derivative of this spline provides a  score estimator $\hat{s}(\cdot)$. Figure~\ref{fig:Bayes_BESQ} presents empirical Bayes estimation curve $f(\hat{u},z)=(2\hat{s}(z)+1)\sqrt{z}$ for the Gamma example data, along with the actual $\{(z_i,f(u_i,z_i))\}_{i=1}^N$ plotted. One can see that the points are closely centered around the estimated curve, even for large and small values of $z_i$, indicating the effectiveness of the empirical Bayes approach for observations under noncentral chi-squared noise.

\begin{figure}[!h]
    \centering
    \caption{Empirical Bayes Estimation Curves for BESQ}
    \label{fig:Bayes_BESQ}
        \includegraphics[width=0.55\textwidth]{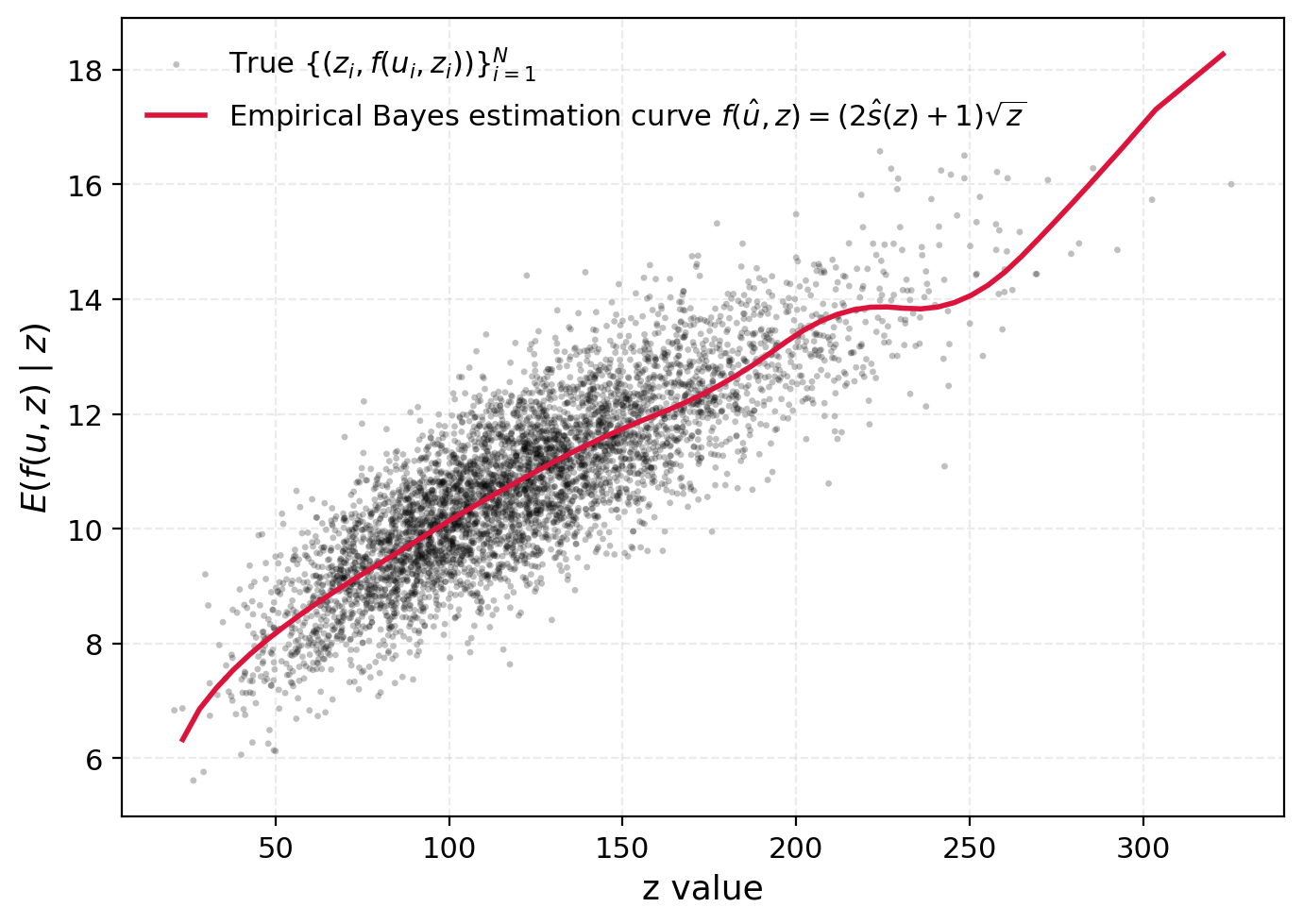}
\end{figure}

\section{Conclusion}\label{sc5}
\quad We derive analytically Tweedie’s formulae for several important non-Gaussian processes, including geometric Brownian motion, squared Bessel processes and CIR processes, which directly yields denoising score matching objectives for diffusion models based on the corresponding processes. Empirically, we employ Tweedie’s formulas derived from GBM and CIR processes for MNIST image generation and financial time series modeling, demonstrating the effectiveness of the resulting denoising score-matching methods. We also apply Tweedie’s formula under the BESQ framework to estimate the noncentrality parameter from the noncentral chi-squared noise for empirical Bayes estimation. To sum, we extend original Tweedie’s formula to non-Gaussian processes and showcase the promise of non-Gaussian diffusion models for practical applications, opening the gate to their adoption in task-specific domains.

\bigskip
{\bf Acknowledgment.}
Tang is supported by
NSF CAREER Award DMS-2538791, the Tang Family Assistant Professorship
and a Columbia-CityU/HK collaborative project that is supported by InnoHK Initiative, The Government of the HKSAR and the AIFT Lab.
Touzi is partially supported by NSF grant DMS-2508581.
Zhang is supported by Tang Fellowship.
Zhou acknowledges financial supports through the Nie Center for Intelligent Asset Management at Columbia.

\appendix
\section{Empirical Bayes Estimation Based on GBM}\label{app:gbm}
\quad Suppose some large number $N$ of possibly correlated log normal variates $z_i>0$ have been observed, each with its own unobserved parameter $u_i$:
$$
z_i\sim \mathrm{LogNormal}(u_i,\sigma^2),\quad i=1,2,\cdots,N.
$$
We estimate the corresponding $u_i$ values through an empirical Bayes approach. Suppose that $u$ has been sampled from a prior distribution $g(\cdot)$, and  $z\,|\,u\sim \mathrm{LogNormal}(u,\sigma^2)$ observed with $\sigma^2$ known:
$$
u_1,\cdots,u_N\sim g(\cdot), \quad \text{and}\ z_i\sim \exp(u_i+\sigma \mathcal{N}(0,1)),\quad i=1,\cdots,N.
$$
By Tweedie's formula \eqref{Tweedie_GBM} (in which $\mu=\sigma^2/2$, $t=1$, $X_0=u$, and $X_1=z$ are specified), we obtain the posterior expectation of $u$ given $z$ as
$\mathbb{E}(u\,|\, z)=\sigma^2 z \cdot s(z)+\sigma^2+\log z$,
where $s(\cdot)$ denotes the score function of the marginal density $p(\cdot)$ of $z$, i.e.,
$s(z)=\frac{\mathrm{d}}{\mathrm{d}z}\log p(z)$.
 Thus, the corresponding empirical Bayes formula is
$$
\hat{u}_i=\sigma^2 z_i\cdot \hat{s}(z_i)+\sigma^2+\log z_i\quad \text{for}\ i=1,\cdots,N,
$$
where $\hat{s}(\cdot)$ is an estimator to the true score $s(\cdot)$.

\begin{figure}[!h]
    \centering
       \caption{Exponential Example}
    \label{fig:GBM_combined}
    \begin{subfigure}{0.48\textwidth}
        \centering
        \includegraphics[width=0.9\textwidth]{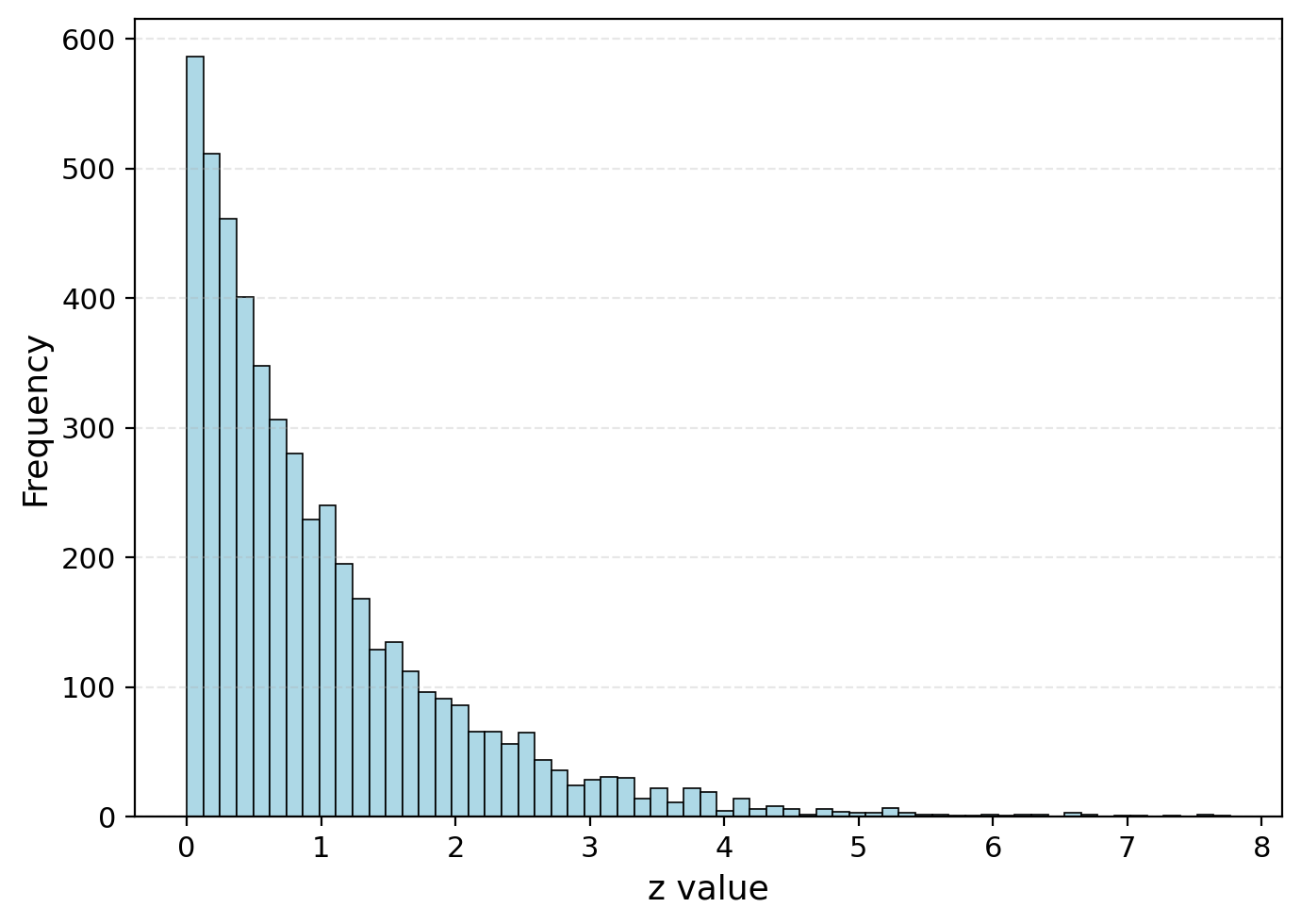}
        \caption{Histogram of $5000\ z_i$'s}
        \label{fig:frequency_GBM}
    \end{subfigure}
    \hfill
    \begin{subfigure}{0.48\textwidth}
        \centering
        \includegraphics[width=0.9\textwidth]{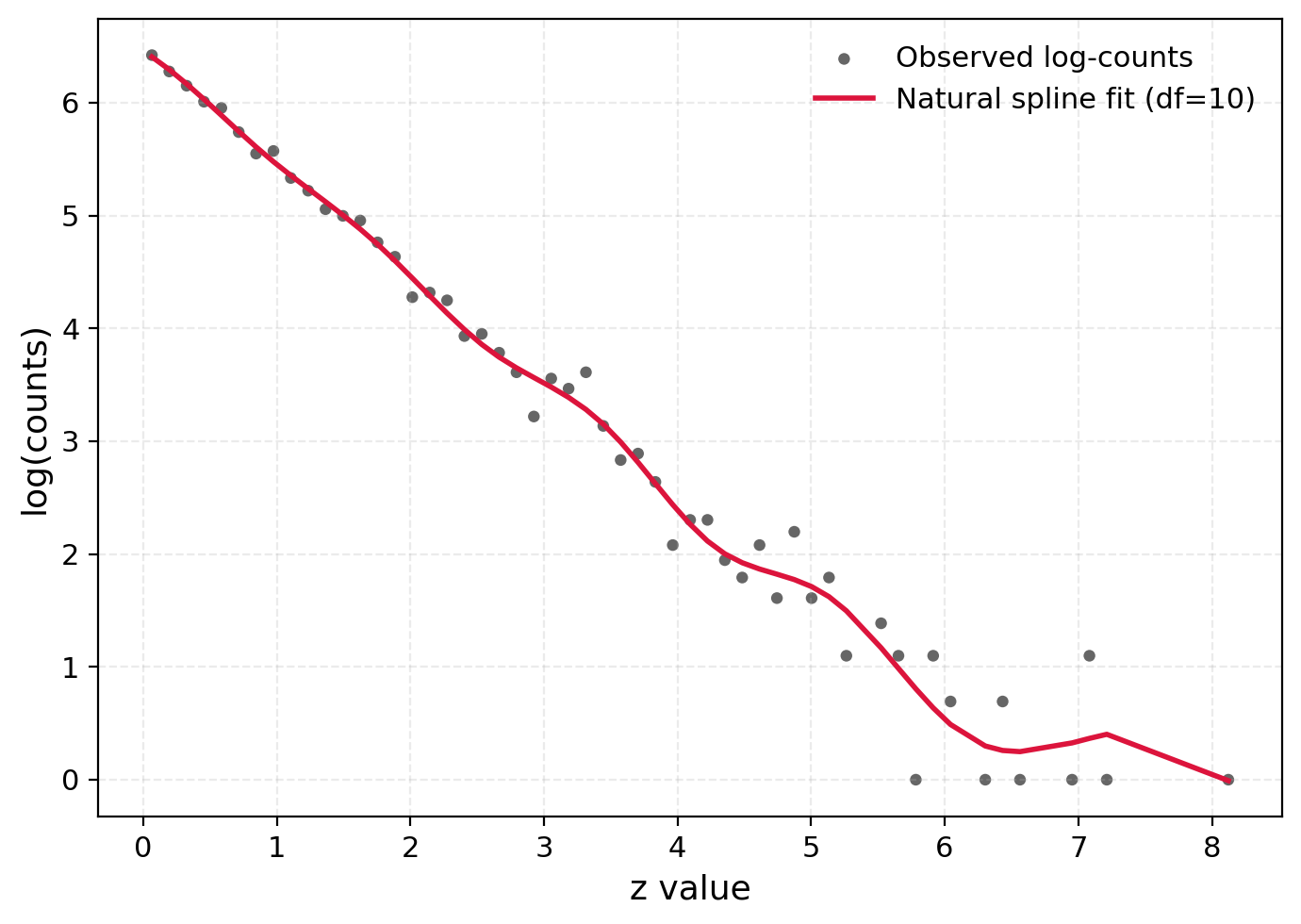}
        \caption{Natural Spline Fit with 10 Degrees of Freedom ($\sigma=0.1$)}
        \label{fig:spline_GBM}
    \end{subfigure}
\end{figure}

\quad Following the experimental setting in \cite{Efron11},
we set $N=5000$ $u_i$ values as $10$ repetitions each of
$$
u_i=\log\log\frac{500}{i-0.5},\quad i=1,\cdots,500.
$$
The empirical distribution of $e^{u_i}$ closely matches an exponential distribution with rate $1$. Then $z_i$ is generated via $\mathrm{LogNormal}(u_i,\sigma^2)$ for each $i$. Here, we take $\sigma=0.1$. Figure~\ref{fig:frequency_GBM} shows the frequency of $5000$ generated $z_i$'s, where there are $63$ bins. Denote the center of the $k$-th bin by $x_k$ and the corresponding bar height by $y_k$ for $k=1,\dots,63$. Figure~\ref{fig:spline_GBM} shows $\log y_k$ against $x_k$, with bins satisfying $y_k=0$ excluded, and a natural spline with 10 degrees of freedom is fitted to the points. The derivative of this spline provides a  score estimator $\hat{s}(\cdot)$. Figure~\ref{fig:Bayes_GBM} presents empirical Bayes estimation curve $\hat{u}(z)=\sigma^2 z\cdot \hat{s}(z)+\sigma^2+\log z$ for the exponential example data with $\sigma=0.1$, along with the actual $\{(z_i,u_i)\}_{i=1}^N$ plotted. One can see that the points are closely centered around the estimated curve, even for large values of $z_i$, indicating the effectiveness of the empirical Bayes approach for observations under log-normal noise.

\begin{figure}[!h]
    \centering
    \caption{Empirical Bayes Estimation Curves ($\sigma=0.1$)}
    \label{fig:Bayes_models}

    \begin{subfigure}[b]{0.45\textwidth}
        \centering
        \includegraphics[width=0.9\textwidth]{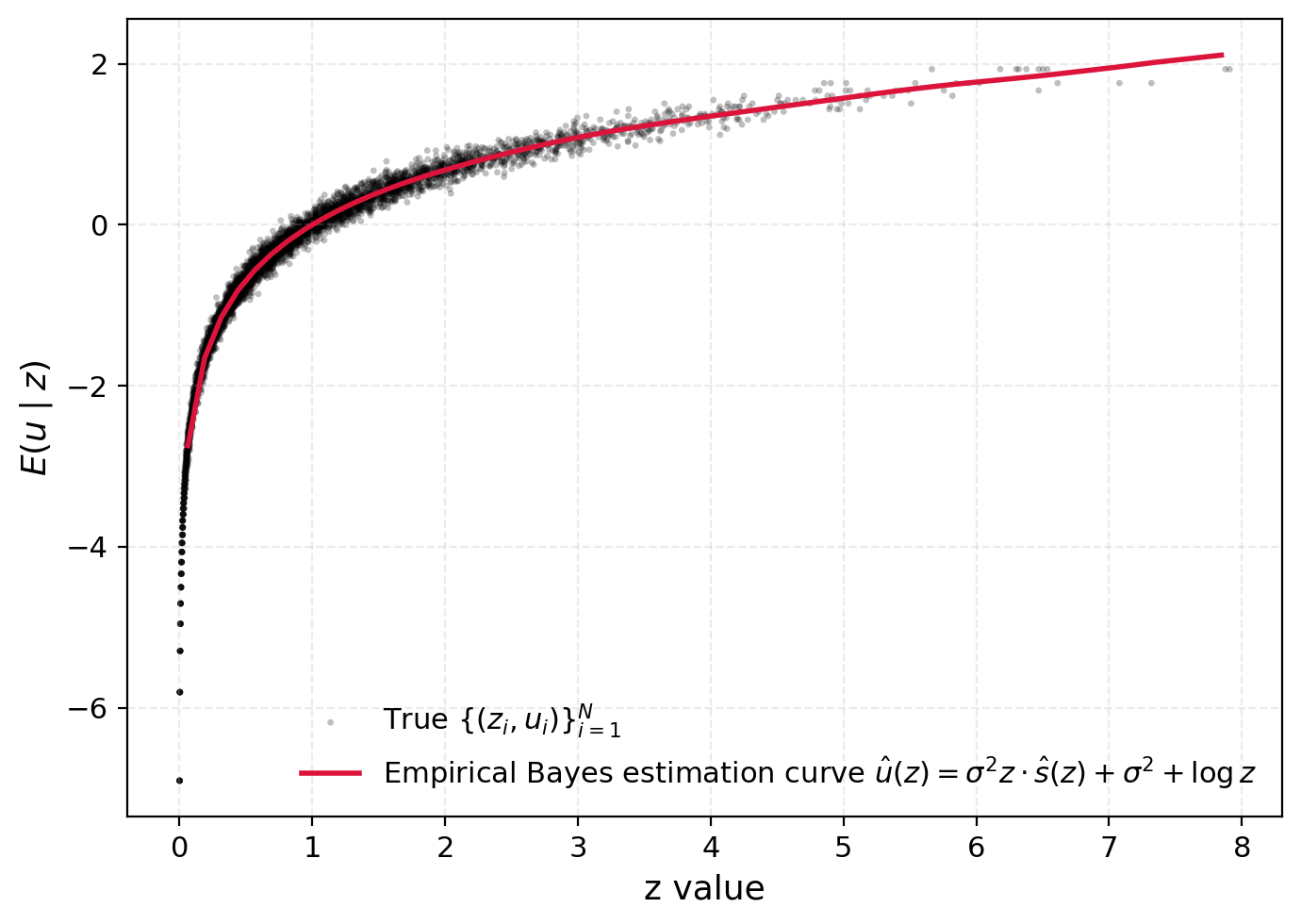}
        \caption{GBM model in $z$-space}
        \label{fig:Bayes_GBM}
    \end{subfigure}
      \hspace{0.01\textwidth}
    \begin{subfigure}[b]{0.45\textwidth}
        \centering
        \includegraphics[width=0.9\textwidth]{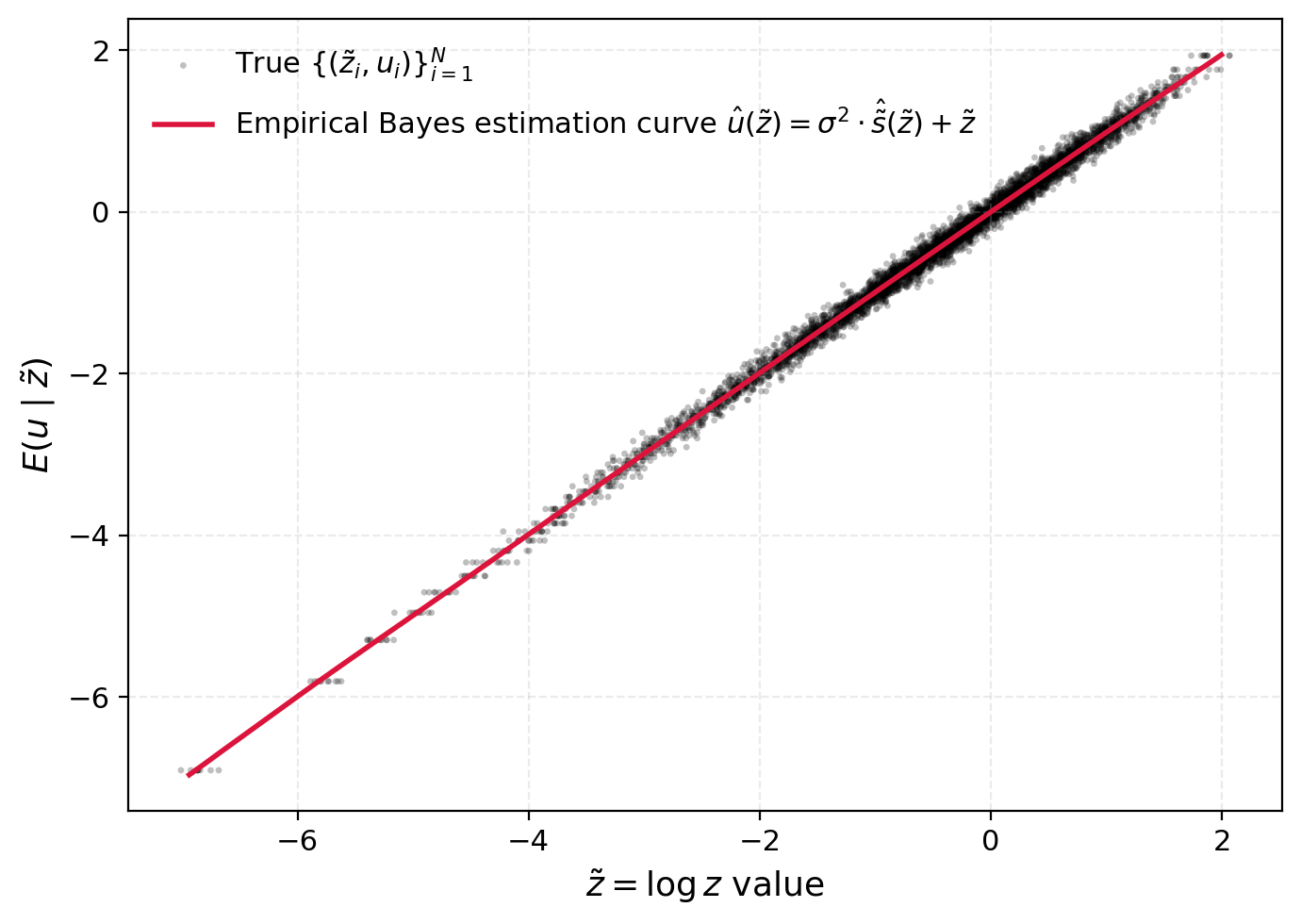}
        \caption{BM model in $(\log z)$-space}
        \label{fig:Bayes_BM}
    \end{subfigure}
\end{figure}

\begin{figure}[!h]
    \centering
    \caption{Empirical Bayes Estimation Curves ($\sigma=0.5$)}
    \label{fig:Bayes_models_0p5}

    \begin{subfigure}[b]{0.45\textwidth}
        \centering
        \includegraphics[width=0.9\textwidth]{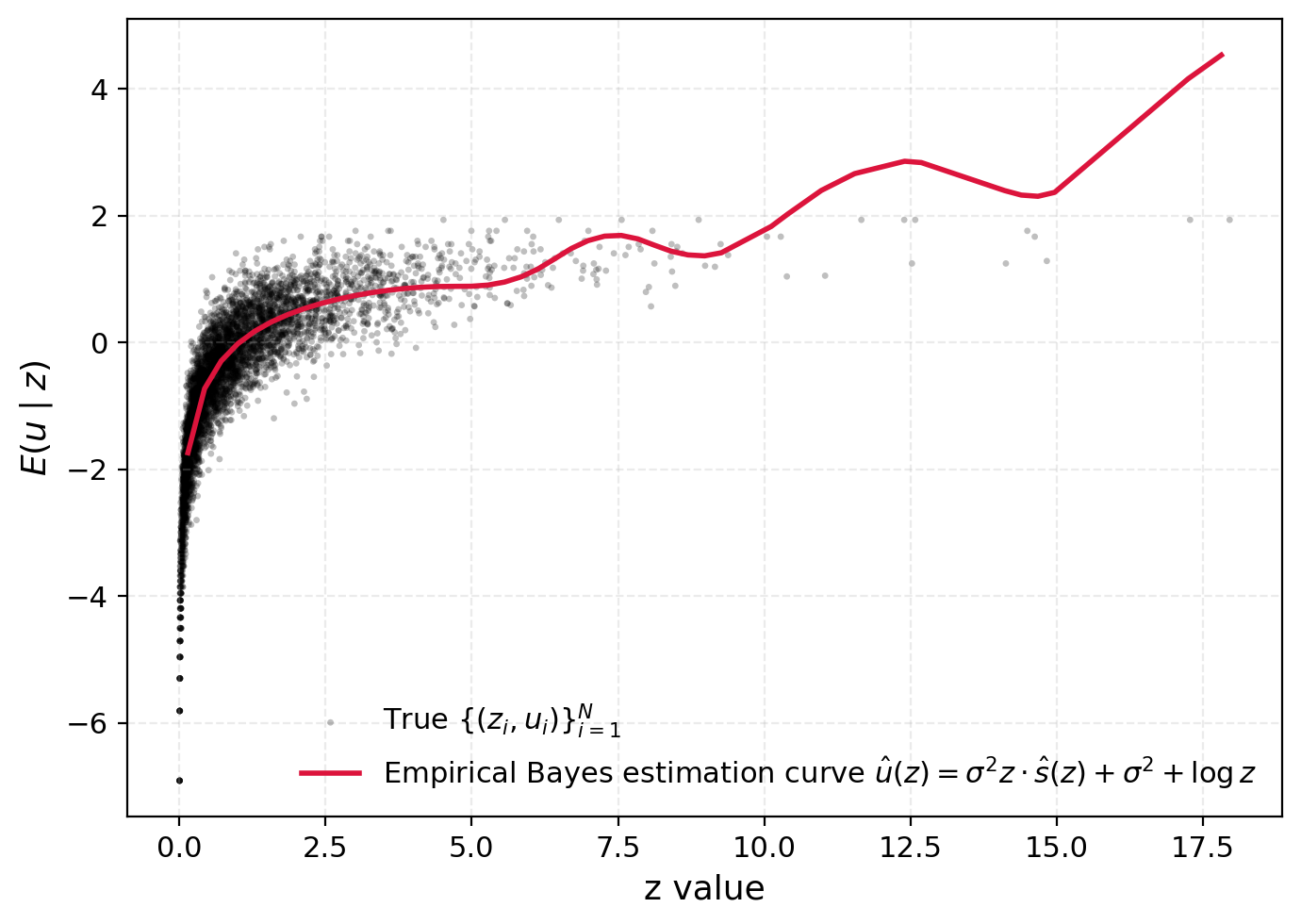}
        \caption{GBM model in $z$-space}
        \label{fig:Bayes_GBM_0p5}
    \end{subfigure}
      \hspace{0.01\textwidth}
    \begin{subfigure}[b]{0.45\textwidth}
        \centering
        \includegraphics[width=0.9\textwidth]{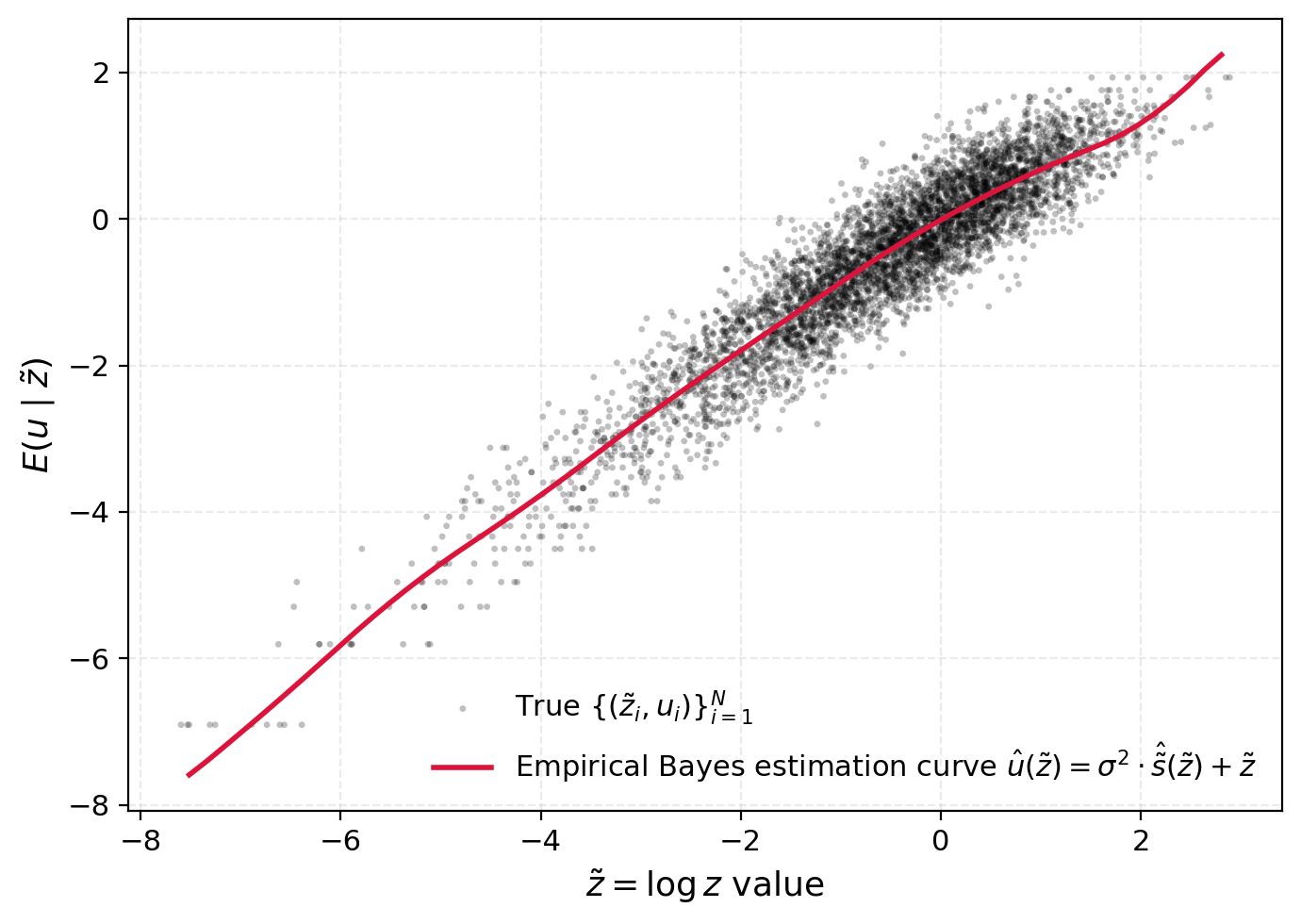}
        \caption{BM model in $(\log z)$-space}
        \label{fig:Bayes_BM_0p5}
    \end{subfigure}
\end{figure}

\begin{figure}[!h]
    \centering
    \caption{Empirical Bayes Estimation Curves ($\sigma=1.0$)}
    \label{fig:Bayes_models_1}

    \begin{subfigure}[b]{0.45\textwidth}
        \centering
        \includegraphics[width=0.9\textwidth]{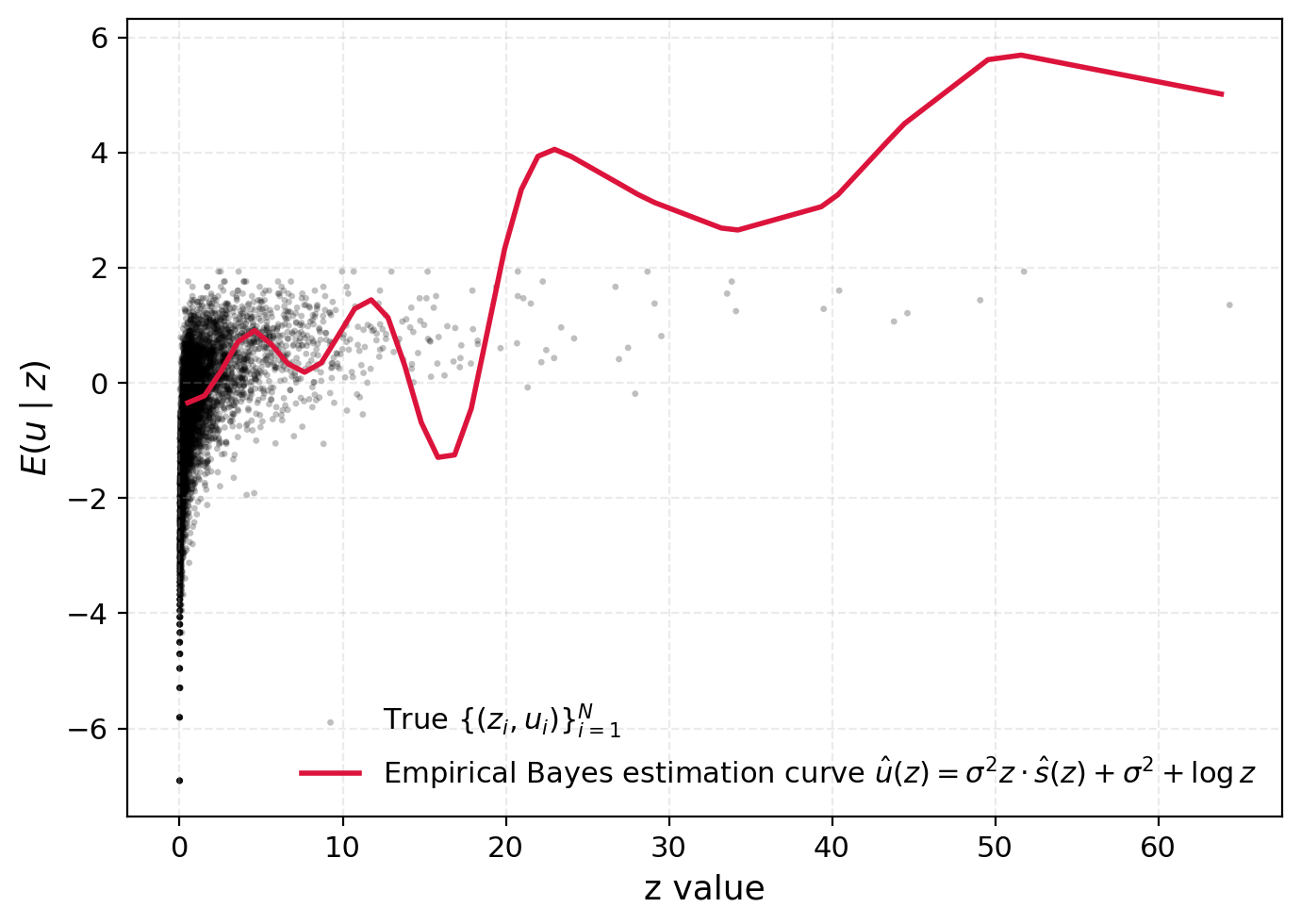}
        \caption{GBM model in $z$-space}
        \label{fig:Bayes_GBM_1}
    \end{subfigure}
      \hspace{0.01\textwidth}
    \begin{subfigure}[b]{0.45\textwidth}
        \centering
        \includegraphics[width=0.9\textwidth]{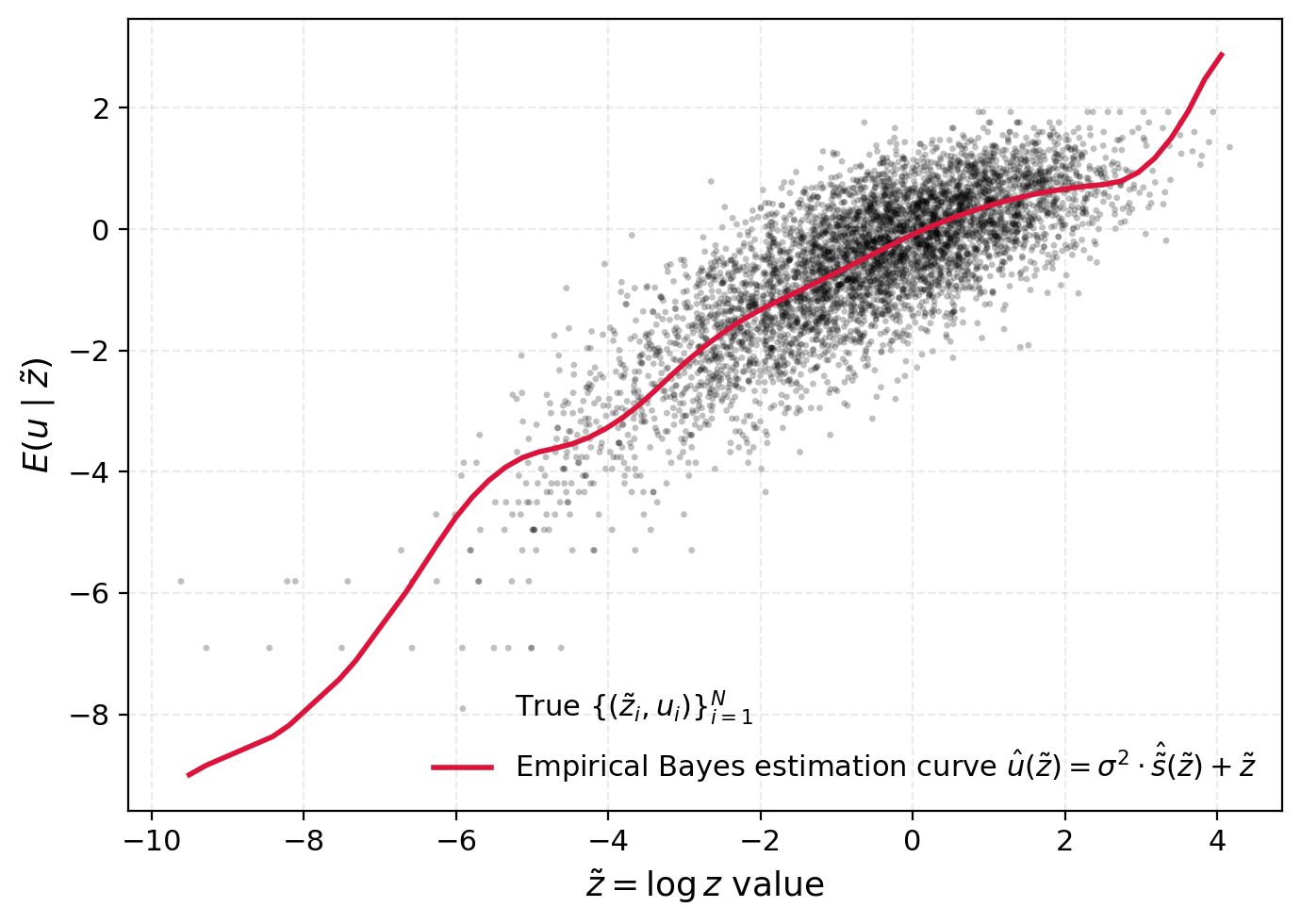}
        \caption{BM model in $(\log z)$-space}
        \label{fig:Bayes_BM_1}
    \end{subfigure}
\end{figure}

\quad We also consider directly applying Tweedie's formula in the $(\log z)$-space to estimate $u_i$. Let $\tilde{z}_i = \log z_i$ for all $i \in [N]$. Then, equivalently, we observe:
$$
\tilde{z}_i\sim \mathcal{N}(u_i,\sigma^2),\quad i=1,\cdots,N,
$$
and aim to estimate $u_i$. Therefore, we can apply Tweedie's formula for the Gaussian case in the $\tilde{z}$-space to obtain an estimate of $u_i$:
$$
\tilde{u}_i=\sigma^2\cdot \hat{\tilde{s}}(\tilde{z}_i)+\tilde{z}_i\quad \text{for}\ i=1,\cdots,N,
$$
where $\tilde{s}(\cdot)$ is the score function of the marginal of $\tilde{z}$ and $\hat{\tilde{s}}(\cdot)$ is an estimator to $\tilde{s}(\cdot)$. Estimating $\tilde{s}(\cdot)$ from $(\tilde{z}_i)_{i=1}^N$ is similar to the procedure described earlier for estimating $s(\cdot)$ from $(z_i)_{i=1}^N$. The estimation results of both models are presented in Figures~\ref{fig:Bayes_models}, \ref{fig:Bayes_models_0p5}, and \ref{fig:Bayes_models_1} for $\sigma=0.1, 0.5$, and $1.0$, respectively. We observe that when $\sigma$ is larger, the $z$ values can take more extremely small or large values due to the exponential Gaussian noise multiplied on $u$, which makes fitting a smooth spline directly in the $z$-space difficult and can lead to substantial estimation errors, particularly for large and small $z$. By contrast, in the $(\log z)$-space, Gaussian noise is added to $u$ and thus the points are more compactly distributed, making the spline easier to fit. Therefore, for relatively large values of $\sigma$ (e.g., $\sigma \geq 0.3$), it is preferable to use Tweedie’s formula based on the BM model rather than the GBM model for empirical Bayes estimation. On the other hand, for small $\sigma$ (e.g., $0 < \sigma \leq 0.3$), both BM and GBM models perform adequately. When $\sigma$ is very large (e.g., $\sigma \geq 2$), both models yield poor estimates for both large and small $z$ values due to the high noise variance.

\section{Experimental Details}\label{app:experiment}
\quad This section provides implementation details of the experiments reported in Section~\ref{sc4-1}.
Our score networks are parameterized using a U-Net architecture \cite{ronneberger2015u} with approximately $1$ million parameters, adapted from the implementation available at \url{https://colab.research.google.com/drive/120kYYBOVa1i0TD85RjlEkFjaWDxSFUx3?usp=sharing}. For training, we employ the Adam optimizer with a learning rate of $1\times 10^{-4}$, along with an exponential moving average with a decay rate of 0.9999.

\quad
For the MNIST image generation in Section~\ref{sec:mnist},
we choose the model parameters:
\begin{itemize}
\item
$\sigma(t) = 0.01+1.99t^{3/2}$, $\mu(t) = \sigma(t)^2/2$, $T = 1$, $a = 1.5$, and $b = 0.5$ for GBM (Figure~\ref{fig:generated_MNIST_GBM});
\item $\sigma(t)= 25^t$, $T=1$, $a=0$, and $b=1$ for VE (Figure~\ref{fig:generated_MNIST_BM});
\item
$\alpha(t)=0.05+4.95t$, $\mu(t)\equiv 1$, $T=1$, $a = 0.5$, and $b = 1.5$ for the CIR process (Figure~\ref{fig:generated_MNIST_CIR});
\item $\alpha(t)=0.05+9.95t$, $T=1$, $a=0$, and $b=1$ for VP (Figure~\ref{fig:generated_MNIST_VP}).
\end{itemize}
Samples are generated by the Euler--Maruyama scheme with 1000 uniform denoising steps.
For the financial data generation in Section~\ref{sec:fin}, we choose:
\begin{itemize}
\item $\sigma(t)=0.001+1.999t$, $\mu(t) = \sigma(t)^2/2-0.25$, and $T = 1$ for GBM (Figure~\ref{fig:gbm});
\item $\alpha(t)=0.05+1.575t$ and $T=1$ for VP (Figure~\ref{fig:vp}).
\end{itemize}
Samples are generated by Euler--Maruyama with 500 uniform denoising steps.

\bibliographystyle{abbrv}
\bibliography{unique}

\end{document}